\documentclass{article}
\PassOptionsToPackage{table}{xcolor}
\usepackage{conference_preprint,times}
\preprintfinalcopy

\usepackage{amsmath}
\usepackage{amssymb}
\usepackage{amsfonts}
\usepackage{amsthm}
\theoremstyle{plain}
\newtheorem{proposition}{Proposition}
\theoremstyle{definition}
\newtheorem{definition}{Definition}
\usepackage{bm}
\usepackage{booktabs}
\usepackage{multirow}
\usepackage{xspace}
\usepackage{graphicx}
\usepackage{algorithm}
\usepackage{algorithmic}
\usepackage{hyperref}
\hypersetup{hidelinks}
\usepackage{url}
\usepackage{placeins}
\usepackage{tikz}

\usepackage{tabularx}
\usepackage{array}
\usepackage{makecell}

\newcommand{\aucres}[2]{%
  \ensuremath{#1\,{\scriptstyle\pm\,#2}}%
}
\newcommand{\bestauc}[2]{%
  \ensuremath{\mathbf{#1}\,{\scriptstyle\pm\,#2}}%
}
\newcommand{\secondauc}[2]{%
  \ensuremath{\underline{#1}\,{\scriptstyle\pm\,#2}}%
}
\newcolumntype{I}{>{\hsize=.875\hsize\linewidth=\hsize\centering\arraybackslash}X}
\newcolumntype{R}{>{\hsize=1.0\hsize\linewidth=\hsize\centering\arraybackslash}X}
\newcolumntype{C}{>{\hsize=1.5\hsize\linewidth=\hsize\centering\arraybackslash}X}

\definecolor{linkcol}{RGB}{46,111,176}
\newcommand{\codeicon}{\raisebox{-0.22em}{\includegraphics[height=1.05em]{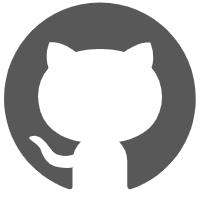}}}
\newcommand{\globeicon}{\raisebox{-0.22em}{\includegraphics[height=1.05em]{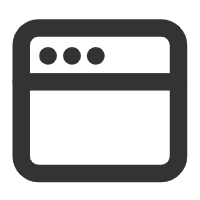}}}
\newcommand{\videoicon}{\raisebox{-0.22em}{\includegraphics[height=1.05em]{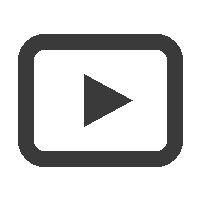}}}
\newcommand{\iconlink}[3]{\href{#1}{#2\,\,{\color{linkcol}\textbf{#3}}}}

\makeatletter
\def\ps@plain{%
  \let\@mkboth\@gobbletwo
  \let\@oddhead\@empty
  \let\@evenhead\@empty
  \def\@oddfoot{\hfil\thepage\hfil}
  \def\@evenfoot{\hfil\thepage\hfil}
}
\makeatother
\newcommand{\coax}{\textsc{CoAx}\xspace}
\newcommand{\comp}{\mathrm{comp}}
\newcommand{\energy}{\mathcal{E}}
\newcommand{\dz}{\delta z}

\newcommand{\Sset}{\mathcal{S}}
\newcommand{\Uset}{\mathcal{U}}
\newcommand{\best}[1]{\textbf{\boldmath #1}}
\newcommand{\secnd}[1]{\underline{#1}}
\definecolor{hdr}{RGB}{226,236,245}
\definecolor{ourrow}{RGB}{238,244,250}
\definecolor{ourred}{RGB}{46,111,176}
\definecolor{primarycolor}{HTML}{A85F1A}
\definecolor{backupcolor}{HTML}{2F6FAE}
\definecolor{irrelevantcolor}{HTML}{565656}
\newcommand{\primary}[1]{\textcolor{primarycolor}{#1}}
\newcommand{\backup}[1]{\textcolor{backupcolor}{#1}}
\newcommand{\irrelevant}[1]{\textcolor{irrelevantcolor}{#1}}

\title{Conditional Co-Ablation: Recovering\\ \mbox{Self-Repair} Backups in Transformer Circuits}

\author{%
Zhiren Gong$^{1,2}$ \quad
He Lu$^{3}$ \quad
Tiantong Wang$^{1}$ \quad
Yichi Zhang$^{2,5}$ \quad
Yixin Wang$^{1}$ \\
\bf Zihao Zeng$^{1}$ \quad
Ming Xiao$^{6}$ \quad
Chau Yuen$^{4}$ \quad
Wei Yang Bryan Lim$^{1}$\\
$^{1}$College of Computing and Data Science, Nanyang Technological University, Singapore\\
$^{2}$Interdisciplinary Graduate Programme, Nanyang Technological University, Singapore\\
$^{3}$School of Civil and Environmental Engineering, Nanyang Technological University, Singapore\\
$^{4}$School of Electrical and Electronic Engineering, Nanyang Technological University, Singapore\\
$^{5}$The Lee Kong Chian School of Medicine, Nanyang Technological University, Singapore\\
$^{6}$Information Science \& Engineering, KTH Royal Institute of Technology, Sweden\\
{\small\ttfamily
zhiren001@e.ntu.edu.sg\hspace{0.7em}
zihao.zeng@ntu.edu.sg\hspace{0.7em}
bryan.limwy@ntu.edu.sg
}
}

\begin{document}
\maketitle

\thispagestyle{plain}

\vspace{-10mm}
\begin{center}\vspace{-2pt}
{\hypersetup{hidelinks}%
\iconlink{https://github.com/GongZhiren/Conditional-Co-Ablation}{\codeicon}{Code}\hspace{20pt}%
\iconlink{https://gongzhiren.github.io/Conditional-Co-Ablation-website}{\globeicon}{Project Page}\hspace{20pt}%
\iconlink{https://gongzhiren.github.io/Conditional-Co-Ablation-website/tutorial.html}{\videoicon}{Tutorial}}
\end{center}
\vspace{2mm}

\begin{abstract}
Mechanistic interpretability seeks to explain transformer behavior through \emph{circuits}: sets of internal components that causally support a behavior.
However, \emph{self-repair} creates a blind spot: ablating a \emph{primary} component can activate a dormant \emph{backup}, so a circuit that explains behavior in the intact model can become incomplete under the intervention used to test it.
We formulate this gap as \emph{conditional circuit completion}: given a primary set, identify components that become causally important after its removal.
We introduce \emph{conditional co-ablation} (\coax), which ranks candidates by growth in ablation effect after primary-set removal.
We show that a perfectly dormant backup can be indistinguishable from an irrelevant component to per-unit intact-state scores, whereas its conditional effect change exactly aggregates all interaction orders linking it to the removed set.
On GPT-2-small's Indirect Object Identification (IOI) circuit, \coax recovers the documented backup heads at $0.941$ ROC-AUC, versus $0.665$ for the strongest intact-state attribution baseline and $0.758$ for the matched conditional-energy control.
Recovery drops to $0.40{\pm}0.13$ AUC for alternative component sets matched in behavioral effect, output displacement, and depth, showing that recovery is specific to the removed circuit.
Beyond recovery, the \coax-selected heads are causally load-bearing: freezing them after primary removal sharply reduces the IOI margin, while adding them to the incomplete circuit reduces incompleteness from $0.47$ to $0.24$.
More broadly, conditional growth aligns with intervention-derived repair in $11/12$ held-out instances across 4 mechanism clusters, and \coax completions outperform matched random completions on all 8 non-GPT-2 models spanning 6 architecture families.
Together, causal explanations of self-repairing transformers must account for backup circuitry when primary components fail.
\end{abstract}

\section{Introduction}
Mechanistic interpretability seeks to explain transformer behavior through internal mechanisms, enabling transparent and auditable accounts of model behavior~\citep{bereska2024mechanistic}.
A central object in such explanations is a \emph{circuit}: a set of internal components that causally contribute to a model's behavior~\citep{olah2020zoom,elhage2021framework}.
To identify such circuits, prior work commonly intervenes on individual components---most directly through \emph{ablation}---and measures how the model's output changes.
Gradients and related approximations scale this principle to rank components and assemble circuits~\citep{wang2022ioi,conmy2023acdc,syed2023attribution,hanna2024faith,kramar2024atp,marks2024sparse}.
These methods have been effective at identifying the circuitry that supports behavior in the original, unablated model.

However, \emph{self-repair} in transformers challenges this picture: ablating a circuit component can change which circuit carries the behavior.
In the Hydra effect~\citep{mcgrath2023hydra,rushing2024selfrepair}, ablating a \primary{\emph{primary}} component can activate a dormant \backup{\emph{backup}} that takes over its role.
Standard intact-state attribution can then underestimate both components: the primary because its loss is compensated, and the backup because it has not yet acted in the original model.

The Indirect Object Identification (IOI) circuit in GPT-2-small provides a concrete example.
Given ``Mary and John went to the store; John gave a drink to \underline{\hspace{1.1em}},'' three \primary{\emph{primary}} name-mover heads help write the correct name, ``Mary,'' to the output~\citep{wang2022ioi}.
Ablating these three heads changes the correct-versus-competing-name logit margin by only $0.28$ from its original $2.57$, because eight documented \backup{\emph{backup}} name-movers compensate for their removal.
Yet ablating the primaries together with those backups causes $1.9\times$ as much damage as the sum of their individual effects.
Thus, the backup branch is a causally important part of how the model sustains the behavior under ablation.

\begin{figure}[!t]\centering
\includegraphics[width=0.95\textwidth]{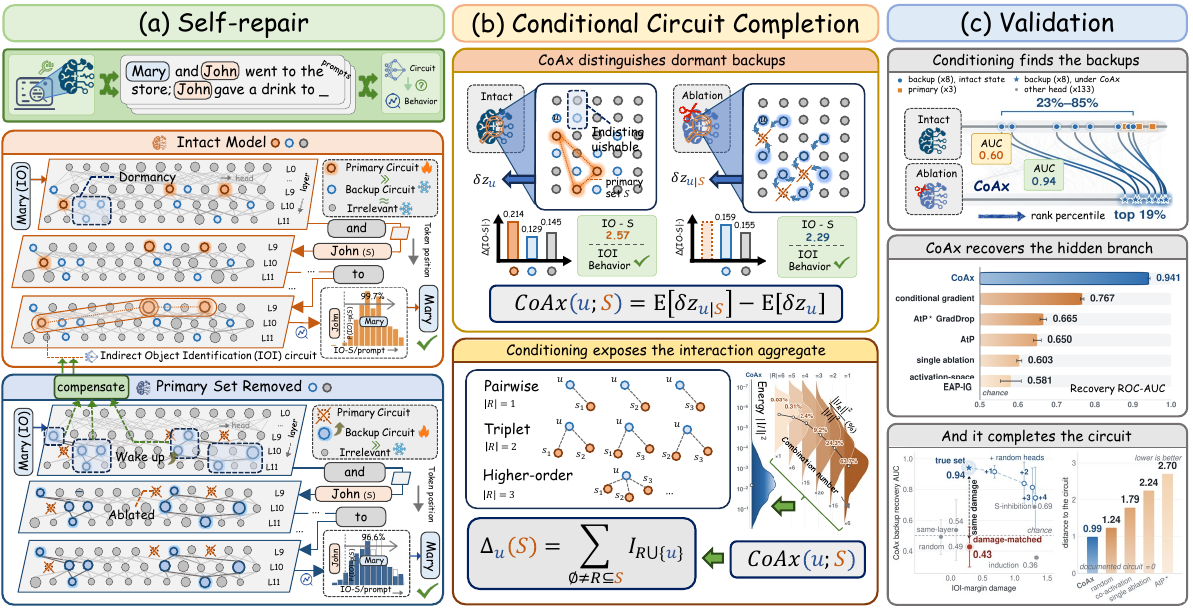}
\vspace{-1mm}
\caption{\textbf{From self-repair to conditional circuit completion.}
\textbf{(a)} Ablating the primary IOI circuit recruits a dormant backup branch while largely preserving behavior.
\textbf{(b)} \coax compares each candidate's ablation effect with the primary set present and removed; this contrast distinguishes dormant backups and exposes the all-order effect-change aggregate linking each candidate to the removed set.
\textbf{(c)} On GPT-2-small IOI, conditioning exposes the documented backups, improves their recovery over intact-state baselines, and substantially completes the missing backup branch.}
\label{fig:teaser}
\vspace{-4mm}
\end{figure}

These observations expose a gap in causal circuit explanation: a circuit can faithfully explain the intact model yet become incomplete under the intervention used to test it.
This raises a \emph{conditional question}: which components become causally important only after a known part of the circuit is removed?

We formalize this problem as \emph{conditional circuit completion}.
Given a \primary{primary} set $\Sset$ supplied by the analyst, the goal is to identify components whose causal importance emerges after $\Sset$ is removed.
Our method, \emph{conditional co-ablation} (\coax), compares each component's ablation effect with $\Sset$ present and removed, and ranks candidates by how much that effect grows.
A dormant \backup{backup} therefore has little effect while $\Sset$ is present but a large effect after $\Sset$ is removed, whereas an \irrelevant{irrelevant} component has little effect in either condition.

This conditional comparison also has a precise theoretical interpretation.
Per-unit information measured only in the intact model can be insufficient to distinguish a perfectly dormant backup from an irrelevant component (Proposition~1).
By contrast, the change in a component's ablation effect when $\Sset$ is removed exactly aggregates all interaction orders linking that component to $\Sset$ (Proposition~2).
Thus, conditioning on the primary set exposes the interactions through which dormant backups become causally important.

On the documented GPT-2-small IOI circuit, \coax recovers the documented backup branch at $0.941$ ROC-AUC, outperforming both the strongest intact-state attribution baseline ($0.665$) and the matched conditional-energy control ($0.758$).
Replacing the primary set with matched alternative component sets drops backup recovery to about $0.40$ AUC, confirming that recovery is specific to the removed primary set.
Beyond recovery, freezing the recovered heads at their intact-state outputs after primary removal sharply degrades the IOI margin, establishing that the selected set is causally load-bearing.
At the circuit level, the same raw \coax ranking reduces circuit incompleteness from $0.47$ to $0.24$.
Beyond IOI, signed growth remains positively aligned with intervention-derived repair in $11/12$ held-out instances across four mechanism clusters, and \coax-based circuit completion outperforms size-matched random completion on all eight non-GPT-2 models spanning six architecture families.
Together, these results show that \coax recovers primary-set-specific backups, identifies components that causally carry repair, and extends conditional circuit completion beyond a single annotated circuit.

\textbf{Contributions.}
\textbf{(i)~Problem and method}: we formulate \emph{conditional circuit completion} and introduce \coax, which ranks components by growth in ablation effect after a primary set is removed (Section~\ref{sec:score}).
\textbf{(ii)~Theory}: we show that per-unit intact-state information can fail to distinguish dormant backups from irrelevant components, and characterize the effect change after removal as aggregated interactions with the removed set (Section~\ref{sec:theory}).
\textbf{(iii)~Evidence}: on IOI, \coax recovers the documented backup branch, establishes recovery specificity to the removed primary set, verifies its causal role in self-repair, and substantially closes the circuit's incompleteness gap (Sections~\ref{sec:discovery}--\ref{sec:apps}).
\textbf{(iv)~Generalization}: beyond the annotated IOI backups, we validate conditional growth against intervention-derived repair across four mechanism clusters and show conditional circuit completion on eight non-GPT-2 models spanning six architecture families (Section~\ref{sec:panel}).

\section{Conditional Co-Ablation: Importance Under Intervention}
\label{sec:method}

We consider frozen decoder-only transformers, where a \emph{structured unit} is any component with an additive residual-stream contribution that can be zeroed by ablation~\citep{elhage2021framework}.
Let $\Sset$ be the supplied primary set and $\Uset$ the candidate structured units for conditional completion, $\Uset\cap\Sset=\emptyset$; in our experiments, these units are attention heads, matching the granularity of the documented IOI primary and backup heads~\citep{wang2022ioi}.
For an input--position pair $(x,t)$, ablating a set $M$ of units yields vocabulary logits $z_M(x,t)$; we suppress $(x,t)$ when unambiguous.
We write $z_\emptyset\equiv z_0$ for the clean logits and $p_0=\mathrm{softmax}(z_0)$ for the corresponding clean output distribution.

\subsection{The conditional co-ablation score}
\label{sec:score}

For a candidate $u\in\Uset$, ablating $u$ on the original model gives the effect
$\dz_u=z_0-z_{\{u\}}$, the change from the clean logits $z_0$ to the logits after ablating $u$.
With the primary set $\Sset$ removed, $z_\Sset$ becomes the reference logits, and ablating $u$ in addition gives

\begin{equation}
\dz_{u\mid\Sset}
=
z_\Sset-z_{\Sset\cup\{u\}} .
\label{eq:condeff}
\end{equation}
Thus, $\dz_u$ measures $u$ with $\Sset$ present, whereas
$\dz_{u\mid\Sset}$ measures it with $\Sset$ removed.
A dormant \backup{backup} changes the output little while $\Sset$ is present but much more after $\Sset$ is removed, whereas an \irrelevant{irrelevant} unit has little effect in either condition.
Let $\energy(\cdot)$ denote the label-free measure of ablation-effect magnitude defined in Eq.~\ref{eq:energy}.
We define the \emph{conditional growth} of this magnitude as:

\begin{definition}[Conditional co-ablation score]
\label{def:coax}
\begin{equation}
\comp_u(\Sset)
=
\energy\!\big(\dz_{u\mid\Sset}\big)
-
\energy\!\big(\dz_u\big).
\label{eq:comp}
\end{equation}
\end{definition}

\noindent Thus, $\comp_u(\Sset)$ measures how much the magnitude of $u$'s ablation effect grows when the primary set is removed.
The reference logits $z_0$ and $z_\Sset$ are shared across candidates; scoring all $|\Uset|$ candidates requires two shared reference evaluations and two candidate-specific evaluations per candidate, hence $O(|\Uset|)$ forward evaluations.
No backward pass or task gradient is required.

If removing $\Sset$ leaves $u$'s effect unchanged, then
$\dz_{u\mid\Sset}=\dz_u$ and $\comp_u(\Sset)=0$, even if $u$ is important in both conditions.
Taking the difference removes each candidate's pre-existing importance, isolating the increase in ablation-effect magnitude associated specifically with primary-set removal.

Because the score is signed, $\comp_u(\Sset)>0$ indicates that removing $\Sset$ increases the magnitude of $u$'s ablation effect, while $\comp_u(\Sset)<0$ indicates a decrease.
\coax ranks positive growth, the pattern expected for a backup that becomes important after primary-set removal.

\paragraph{The output measure $\energy$.}
We measure logit perturbations in the Fisher geometry induced by the clean output distribution~\citep{amari1998natural}, so the magnitude reflects local output-distribution sensitivity rather than raw Euclidean logit distance.
For any vocabulary-logit perturbation $v=v(x,t)$, define
\begin{equation}
\widetilde v
=
\sqrt{p_0}\odot
\bigl(v-\langle p_0,v\rangle\mathbf{1}\bigr),
\qquad
\energy(v)
=
\mathbb{E}_{x,t}\!\left[\|\widetilde v\|_2^2\right].
\label{eq:energy}
\end{equation}
Here $\odot$ denotes elementwise multiplication, $\mathbf{1}$ is the all-ones vector, and $\sqrt{p_0}$ is taken elementwise; the expectation is over the evaluated input--position pairs.
Centering by $\langle p_0,v\rangle$ removes the shared logit shift to which softmax is invariant, while the $\sqrt{p_0}$ weighting yields
$\langle\widetilde v,\widetilde w\rangle=v^\top Fw$ for perturbations $v$ and $w$, where
$F=\mathrm{diag}(p_0)-p_0p_0^\top$ is the categorical Fisher matrix
(Appendix~\ref{app:geometry}).
Both terms in Eq.~\ref{eq:comp} use the same clean reference geometry, so their difference reflects a change in ablation-effect magnitude rather than a change in metric.

\subsection{What Conditioning Reveals---and Intact-State Scores Miss}
\label{sec:theory}

Two facts motivate conditioning.
First, per-unit information from the original model can be insufficient to distinguish a perfectly dormant backup from an irrelevant unit, motivating measurement after the primary set $\Sset$ is removed (Proposition~\ref{prop:blind}).
Second, the resulting effect change $\dz_{u\mid\Sset}-\dz_u$ exactly decomposes into interactions linking the candidate $u$ to $\Sset$ (Proposition~\ref{prop:exact}).
Proofs are in Appendix~\ref{app:proofs}.

\begin{definition}[Intact-state score]
\label{def:intactscore}
We call a per-unit score \emph{intact-state} if it is determined only by the candidate's per-unit information measured on the original model without removing $\Sset$, represented here by its activation, output derivative, and single-ablation effect.
Single-ablation saliency and standard local attribution scores are examples of this class.
\end{definition}

\begin{proposition}[Non-identifiability from intact-state information]
\label{prop:blind}
There exist a model, an input distribution, a primary set $\Sset$, and two candidates
$b,n\in\Uset$ with identical intact-state information, yet
$\energy(\dz_b)=\energy(\dz_n)=0$.
After $\Sset$ is removed,
$\energy(\dz_{b\mid\Sset})>0$ while
$\energy(\dz_{n\mid\Sset})=0$.
Every intact-state score assigns $b$ and $n$ the same value, while
$\comp_b(\Sset)>0=\comp_n(\Sset)$ distinguishes them.
\end{proposition}

Proposition~\ref{prop:blind} formalizes this blind spot: a perfectly dormant backup can be indistinguishable from an irrelevant unit on the original model, yet become important after primary-set removal.

To characterize this change exactly, we extend the single-unit effect to any ablated set $M$ by defining
$\dz_M:=z_0-z_M$.
Then $\dz_\emptyset=0$, $\dz_{\{u\}}=\dz_u$, and
$\dz_{u\mid\Sset}
=\dz_{\Sset\cup\{u\}}-\dz_\Sset
=z_\Sset-z_{\Sset\cup\{u\}}$.
Thus, $M\mapsto\dz_M$ is a vector-valued set function.
Using the standard M\"obius transform~\citep{rota1964foundations}, associate each nonempty subset $T\subseteq M$ with the interaction
$I_T:=\sum_{A\subseteq T}(-1)^{|T|-|A|}\dz_A$,
where $A$ ranges over the subsets of $T$.
M\"obius inversion then decomposes the joint ablation effect of $M$ as
\begin{equation}
\dz_M
=
\sum_{\emptyset\neq T\subseteq M} I_T .
\label{eq:mobius}
\end{equation}
The interaction order is $|T|$: $I_{\{u\}}=\dz_u$ is the first-order effect visible from ablating $u$ alone.
For two units,
$I_{\{u_1,u_2\}}=\dz_{\{u_1,u_2\}}-\dz_{u_1}-\dz_{u_2}$
is their second-order non-additivity---the joint effect beyond the single-unit effects; higher orders extend this to larger sets.
We now specialize this decomposition to the interactions linking a candidate $u$ to the removed primary set $\Sset$.

\begin{proposition}[Exact all-order interaction decomposition]
\label{prop:exact}
For any candidate $u\in\Uset$,
\begin{equation}
\dz_{u\mid\Sset}
=
\sum_{R\subseteq\Sset} I_{R\cup\{u\}}
=
\dz_u
+
\sum_{\{s\}\subseteq\Sset} I_{\{u,s\}}
+
\sum_{\{s,s'\}\subseteq\Sset} I_{\{u,s,s'\}}
+\cdots .
\label{eq:exact}
\end{equation}
Hence the \emph{effect change}
$\Delta_u(\Sset):=\dz_{u\mid\Sset}-\dz_u
=\sum_{\emptyset\neq R\subseteq\Sset}I_{R\cup\{u\}}$
is exactly the aggregate of all interaction orders linking $u$ to the removed primary set $\Sset$ and supported on $\Sset \cup \{u\}$.
\end{proposition}

Proposition~\ref{prop:exact} shows that the change in $u$'s ablation effect under primary-set removal is exactly the aggregate of all interaction orders linking $u$ to $\Sset$ and supported on $\Sset \cup \{u\}$; conditioning yields this aggregate without enumerating its terms.
This effect change also gives a direct decomposition of the \coax score.
Since $\dz_{u\mid\Sset}=\dz_u+\Delta_u(\Sset)$, Eq.~\ref{eq:comp} becomes
\begin{equation}
\comp_u(\Sset)
=
\mathbb{E}_{x,t}
\left[
2\left\langle
\widetilde{\dz}_u,
\widetilde{\Delta}_u(\Sset)
\right\rangle
+
\left\|
\widetilde{\Delta}_u(\Sset)
\right\|_2^2
\right].
\label{eq:score_shift}
\end{equation}
Here $\widetilde{\Delta}_u(\Sset)$ applies the transformation in Eq.~\ref{eq:energy} to $\Delta_u(\Sset)$.
For a perfectly dormant unit,
$\energy(\dz_u)=0$ makes the intact-state term vanish, so
Eq.~\ref{eq:score_shift} reduces exactly to
$\energy(\Delta_u(\Sset))$.

Explicitly enumerating all interaction orders is combinatorial in $|\Sset|$.
Even a pairwise truncation requires $O(|\Uset||\Sset|)$ evaluations and omits higher orders, whereas conditioning yields the full interaction aggregate with the $O(|\Uset|)$ cost of Section~\ref{sec:score}.

\paragraph{Why signed growth targets backup recruitment.}
The conditional effect change $\Delta_u(\Sset)$ reveals whether removing $\Sset$ changes $u$'s role, while backup discovery asks the more specific question of whether $u$ becomes more important after the primary set is removed.
Equation~\ref{eq:score_shift} makes this recruitment criterion explicit: $\comp_u(\Sset)>0$ means that the magnitude of $u$'s ablation effect grows after primary-set failure, whereas $\comp_u(\Sset)<0$ means that it shrinks.
For a perfectly dormant backup, the intact-state term vanishes and the score reduces exactly to $\energy(\Delta_u(\Sset))$, the newly exposed conditional effect.
We therefore rank candidates directly by positive signed growth.

Together, these results establish why conditional measurement is necessary and what it reveals: dormant backups can be invisible in the original model, while primary-set removal exposes the all-order interaction aggregate underlying their effect change.
The experiments below first test the signed recruitment signal on a documented dormant backup circuit, then ask whether the same conditional measurement tracks repair beyond that annotated setting.

\section{Recovering the Dormant Backups of a Known Circuit}
\label{sec:discovery}

We test whether \coax recovers a known backup branch, whether that recovery is specific to the primary set removed, and whether the recovered heads causally carry the repair.

\paragraph{Protocol.}
We use the well-characterized GPT-2-small IOI circuit~\citep{radford2019gpt2,wang2022ioi}, where both the \primary{primary} path and its \backup{backup} branch have been independently documented:
the three \primary{primary} name-mover heads form $\Sset$, the eight documented \backup{backup} name-movers are positives, and all remaining heads are candidates.
The backup labels are used only for evaluation; Appendix~\ref{app:noncirc} independently reconstructs a primary-removal-free reference set from intact-state signals.
Full protocol and controls are in Appendix~\ref{app:setup}.

\begin{table}[!ht]
\centering
\small
\renewcommand{\arraystretch}{1.16}
\caption{\textbf{Backup-head discovery on GPT-2-small IOI.}
Head-level ROC-AUC over eight documented backups (four prompt seeds).
Conditional methods use the same supplied primary set; structural and historical readouts are analyzed separately in Appendix~\ref{app:coact}--\ref{app:historical-dla}.}
\label{tab:backup}

\begin{tabularx}{\linewidth}{@{}IIII RRC@{}}
\toprule

\rowcolor{hdr}
\multicolumn{4}{c}{\textbf{Intact only}} &
\multicolumn{2}{c}{\textbf{Removed only}} &
\cellcolor{ourrow}\textbf{Intact $\rightarrow$ removed}
\\

\makecell[c]{Single\\ablation} &
AtP &
EAP-IG &
\makecell[c]{AtP$^\star$\\GradDrop} &
\makecell[c]{Conditional\\gradient} &
\makecell[c]{Conditional\\energy} &
\cellcolor{ourrow}\textbf{\coax}
\\

\midrule

\aucres{0.603}{0.007} &
\aucres{0.650}{0.010} &
\aucres{0.581}{0.027} &
\aucres{0.665}{0.009} &
\secondauc{0.767}{0.005} &
\aucres{0.758}{0.004} &
\cellcolor{ourrow}\bestauc{0.941}{0.004}
\\

\bottomrule
\end{tabularx}
\vspace{-2mm}
\end{table}

\subsection{Backup discovery}
\label{sec:bdisc}

\paragraph{The two-state contrast recovers the backup branch.}
The central result is that \coax recovers the documented backups at $0.941$ ROC-AUC, improving by $0.276$ over the strongest intact-state attribution baseline and by $0.174$ over the strongest removed-state control; against the representation-matched conditional-energy control, the gap is $0.183$.
The same separation is stronger under average precision ($0.557$ versus $0.098$ for AtP$^\star$; Appendix~\ref{app:sig}), supporting the top of the ranking under the $8/141$ class imbalance.
A derived two-state AtP$^\star$ contrast reaches $0.938$ AUC, showing that the conditional principle transfers across estimators; unlike \coax, it is a task-gradient local approximation requiring backward passes rather than an exact finite full-distribution intervention score (Appendix~\ref{app:condatp}).

The reason is simple: the three settings answer different questions.
An intact-state score asks which components already matter; a score read only after primary-set removal asks which components matter in the intervened model,
but still mixes pre-existing importance with newly recruited importance.
\coax instead asks which candidates become more important \emph{because} the primary set was removed.
At the vector level, the corresponding effect change
$\Delta_u(\Sset)=\dz_{u\mid\Sset}-\dz_u$ has the exact decomposition in Proposition~\ref{prop:exact}, aggregating every interaction order linking $u$ to the removed primary set.
The $0.941$ recovery therefore reflects access to the conditional causal information needed to identify a backup branch, rather than simply a stronger importance score. A task-specific historical direct-logit change, closely aligned with how the IOI backup branch was originally characterized, reaches $0.971$ AUC; unlike \coax, it requires the known IO$-$S output direction (Appendix~\ref{app:historical-dla}).

\paragraph{Conditioning carries the main signal; Fisher geometry sharpens it.}
Holding the two-state contrast fixed, replacing the Fisher geometry with plain $\ell_2$ still yields $0.905$ AUC (Figure~\ref{fig:method_analysis}b); the full geometry raises this to $0.941$.
Thus the conditional comparison exposes most of the backup signal, while the Fisher geometry provides a principled refinement of its measurement.

\subsection{Recovery tracks the identity of the removed circuit}
\label{sec:specificity}

To test whether recovery tracks the identity of the removed circuit, we first ask whether other primary sets with similar intervention severity recover the same backups, and then test whether LayerNorm rescaling could explain the signal.

\paragraph{Recovery follows circuit identity, not intervention severity.}
We search $5{,}000$ random three-head sets on calibration prompts and retain the $16$ that best match the true primary set in behavioral damage, primary-set Fisher displacement $\energy(\dz_\Sset)$, and depth.
On held-out prompts, these matched wrong sets recover the backups at only $0.40{\pm}0.13$ AUC, while the true primary set remains at about $0.94$ (Table~\ref{tab:fishermatch}).
This is a hard negative: the interventions are matched both in task-level damage and in the output-space displacement measured in the same Fisher geometry as \coax, yet they fail to expose the backup branch.
Recovery therefore tracks \emph{which circuit} is removed, rather than merely how strongly the model is perturbed.

\paragraph{Freezing normalization sharpens the backup signal.}
Ablation can alter downstream LayerNorm statistics and rescale the remaining computation, a known component of transformer self-repair~\citep{rushing2024selfrepair}.
We remove this channel by replaying ablated forwards with the clean normalization denominators fixed.
Backup AUC rises from $0.941$ to $0.983$: the documented backups retain about $80\%$ of their conditional growth, while non-backup growth is suppressed enough for its mean to shift below zero (Appendix~\ref{app:lnfreeze}).
Thus LayerNorm rescaling does not create the backup signal; removing it instead suppresses background growth and sharpens the separation.

\begin{figure}[!t]
\centering
\includegraphics[width=2.36in]{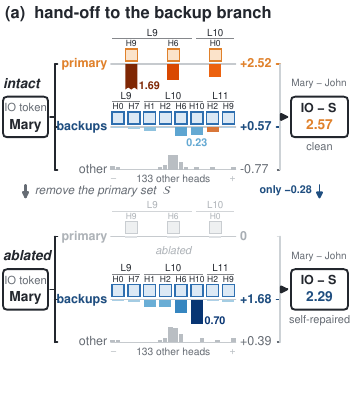}\hfill
\includegraphics[width=1.32in]{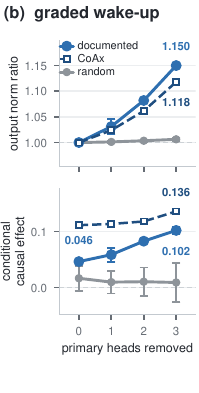}\hfill
\includegraphics[width=1.80in]{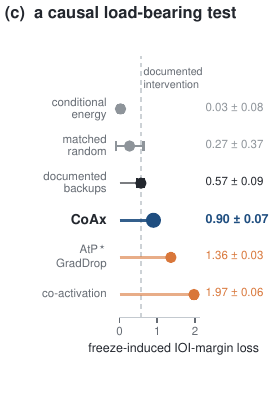}
\vspace{-10mm}
\caption{\textbf{Backup hand-off and causal repair} (GPT-2-small IOI; four prompt seeds).
\textbf{(a)} Removing the primary name-movers shifts direct-logit attribution toward the documented backup branch while the IO$-$S margin remains largely preserved.
\textbf{(b)} The documented backups and the \coax-selected heads both increase on the wake-up readouts as more primaries are removed, while matched random heads remain flat.
\textbf{(c)} Freezing each selector's top-$8$ at its intact-state outputs measures causal load-bearingness; \coax causes a $0.90$-point IOI-margin loss.}
\label{fig:mechanism}
\vspace{-2mm}
\end{figure}

\subsection{Causal validation of the backup hand-off}
\label{sec:realbackups}

Backup-label recovery alone does not establish that the recovered branch actually carries self-repair.
Figure~\ref{fig:mechanism} therefore addresses two stronger questions: does the backup response grow as more of the primary circuit is removed, and does preventing that response remove the repair?

\paragraph{The hand-off is graded and causally load-bearing.}
Removing the \primary{primary} name-movers eliminates their direct-logit contribution, while the documented \backup{backup} branch roughly triples its contribution from $+0.57$ to $+1.68$ and the IO$-$S margin falls only from $2.57$ to $2.29$ (Figure~\ref{fig:mechanism}a).
As progressively more primaries are removed, both the documented backups and the \coax-selected heads increase on the two wake-up readouts, while matched random heads remain flat (Figure~\ref{fig:mechanism}b).
Freezing the documented backups at their intact-state outputs causes a $0.57{\pm}0.09$ IOI-margin loss (Figure~\ref{fig:mechanism}c).
Together, these results show that backup recruitment grows with primary-circuit failure and is causally load-bearing.

\paragraph{\coax's label-free selection is causally load-bearing.}
Applying the same freezing intervention to each selector's top-$8$, the \coax set causes a $0.90{\pm}0.07$ IOI-margin loss, compared with $0.03{\pm}0.08$ for conditional energy alone and $0.27{\pm}0.37$ for a size-matched random set (Figure~\ref{fig:mechanism}c).
The freezing readout establishes causal load-bearingness; circuit completion and knockout then test whether the selected set matches the missing branch's causal role. Co-activation nearly matches \coax on backup-label AUC ($0.93$ versus $0.941$; Appendix~\ref{app:coact}) but gives substantially worse backup-role knockout matching (Section~\ref{sec:apps}).
Thus the freezing, completion, and knockout tests jointly show that \coax identifies an intervention-responsive, causally relevant completion without using backup labels.

\section{Completing the Circuit}
\label{sec:apps}

Section~\ref{sec:discovery} shows that \coax recovers a branch that wakes up and carries self-repair.
Conditional circuit completion asks for more: the recovered components should repair the causal account that was incomplete under intervention.
We use the same raw headline \coax ranking as in discovery---no backup labels, task gradient, or task-direction filter---and take its top eight heads as the completion.
Across the two circuit-level tests, this set reduces incompleteness from $0.47$ to $0.24$ and closely matches the documented backup knockout with little collateral change.

\begin{table}[!ht]
\centering
\small
\caption{\textbf{Circuit completion and backup-role validation}
(GPT-2-small IOI; four prompt seeds).
\textbf{Left:} the raw headline \coax top-$8$ is added to the incomplete primary circuit, with no backup labels, task gradient, or task-direction filter.
\textbf{Right:} each selector's raw top-$8$ knockout is compared with the documented backup intervention.
Lower is better.}
\label{tab:apps}
\vspace{0.5mm}

\begin{minipage}[c]{0.39\linewidth}
\centering
\setlength{\tabcolsep}{3pt}
\renewcommand{\arraystretch}{1.08}
\begin{tabular}{@{}l c@{}}
\toprule
\rowcolor{hdr}
\textbf{circuit} & \textbf{$D(C)\downarrow$} \\
\midrule
primary circuit & $0.474{\pm}0.030$ \\
\rowcolor{hdr!35}
documented circuit & $0.300{\pm}0.027$ \\
\midrule
\rowcolor{ourrow}
+conditional energy & \best{$0.181{\pm}0.027$} \\
\rowcolor{ourrow}
+\coax & $0.236{\pm}0.037$ \\
+amplification ratio & $0.258{\pm}0.035$ \\
+random & $0.499{\pm}0.073$ \\
\bottomrule
\end{tabular}
\end{minipage}
\hfill
\begin{minipage}[c]{0.59\linewidth}
\centering
\setlength{\tabcolsep}{2.5pt}
\renewcommand{\arraystretch}{1.08}
\begin{tabular}{@{}l c c c@{}}
\toprule
\rowcolor{hdr}
\textbf{top-up} &
\makecell[c]{\textbf{margin}\\\textbf{dist.}$\downarrow$} &
\makecell[c]{\textbf{output}\\\textbf{KL}$\downarrow$} &
\makecell[c]{\textbf{unrel.}\\\textbf{KL}$\downarrow$} \\
\midrule
\rowcolor{hdr!35}
documented backups & 0.00 & 0.00 & 0.09 \\
\midrule
\rowcolor{ourrow}
+\coax & \best{0.99} & 0.17 & 0.14 \\
+amplification ratio & \best{0.99} & \best{0.07} & \best{0.09} \\
+conditional energy & 1.71 & 0.46 & 0.47 \\
+co-activation & 1.79 & 0.47 & 0.15 \\
+\coax next-8 & 2.03 & 1.38 & 0.86 \\
+AtP$^\star$ GradDrop & 2.70 & 1.02 & 0.47 \\
+random & 1.24 & 0.47 & 0.24 \\
\bottomrule
\end{tabular}
\end{minipage}
\vspace{-2mm}
\end{table}

\paragraph{The recovered set substantially completes the circuit.}
Following~\citet{wang2022ioi}, completeness asks whether a circuit reproduces
the full model's response to the same ablation.
Let $\Omega$ denote the full head set, and let $C\subseteq\Omega$ denote the retained head support of a circuit for this completeness test.
Let $\mathrm{LD}_C(x)$ be the IOI margin when only heads in $C$ remain active (all others are mean-ablated).
Writing
$r(C)=\mathbb{E}_x[\mathrm{LD}_C(x)-\mathrm{LD}_{C\setminus\Sset}(x)]$
for the response of circuit $C$ to primary-set removal, we define its
incompleteness as
$D(C)=|r(\Omega)-r(C)|$; lower is better.
The controlled primary circuit has $D=0.47$, while the documented circuit
reaches $0.30$.
Adding the raw \coax top-$8$ reduces incompleteness to $0.24$
(Table~\ref{tab:apps}), using exactly the same label-free,
task-direction-free ranking as discovery.
Thus \coax does more than localize backup components: it restores the missing
counterfactual response and turns backup localization into circuit completion.
Conditional energy reaches a lower scalar gap ($0.18$) by selecting generally
large removed-state effects, but its knockout is much farther from the documented
backup intervention; scalar response matching and branch identity are therefore
complementary criteria.

\paragraph{The recovered set reproduces the backup knockout.}
We complement the completion test by reversing the intervention and asking whether the selected knockout reproduces the documented backup effect.
A large performance drop after removal does not by itself identify the backup, because disrupting unrelated computation can also hurt IOI.
We therefore compare each primary-plus-selected removal with the documented primary-plus-backup removal in both the IOI margin and the answer-position output distribution, while tracking collateral KL on unrelated text.
The \coax-selected set closely matches the documented intervention ($0.99$ margin distance and $0.17$ KL), while keeping unrelated-text KL low at $0.14$ (Table~\ref{tab:apps}).
An internal normalized amplification-ratio variant is statistically tied in margin distance and attains lower KL; signed growth remains our default because it admits the exact interaction decomposition in Proposition~\ref{prop:exact} and avoids division by small intact effects.
Together with the completion result above, this closes the circuit-level loop: completion restores the missing counterfactual response, while knockout reproduces the causal role of the missing backup branch. 
Full completion protocol, knockout, and paired-uncertainty analyses are in Appendix~\ref{app:attribution}.

\section{Conditional Repair Beyond Documented Backups}
\label{sec:panel}

Sections~\ref{sec:discovery}--\ref{sec:apps} establish \coax on a documented dormant backup branch.
We now test whether the same conditional signal extends beyond that annotated setting in two ways: against independently measured repair across mechanisms and through held-out circuit completion across model families.

\paragraph{Conditional growth aligns with repair across mechanisms.}
We construct an annotation-free causal target: after ablating a primary set $\Sset$, we freeze each candidate $u$ to its intact-state output, so any further behavioral loss measures repair carried by $u$.
We quantify this by $\mathrm{repair}(u\mid\Sset)=m(\text{ablate }\Sset)-m(\text{ablate }\Sset,\text{ freeze }u)$.
The primary sets follow documented IOI head roles~\citep{wang2022ioi}; we evaluate name-mover and S-inhibition sets on the IOI margin, and induction and duplicate-token sets on repeated-token induction behavior~\citep{olsson2022induction}.
Scores and repair targets use disjoint prompts, making their agreement out of sample.
Across twelve held-out instances from four mechanism clusters, \coax has positive rank agreement with mechanism-matched causal repair in $11/12$ instances and all four clusters, with a hierarchical mean of $+0.123$ and a nested $95\%$ CI of $[+0.047,+0.201]$ (Figure~\ref{fig:regime}a).
Thus the signed-growth signal that recovers the documented IOI backups also shows positive out-of-sample agreement with independently measured repair across distinct mechanisms and prompt forms.
Full protocols and hierarchical inference are in Appendix~\ref{app:panel}.

\begin{figure}[t]
\centering
\includegraphics[width=0.85\textwidth]{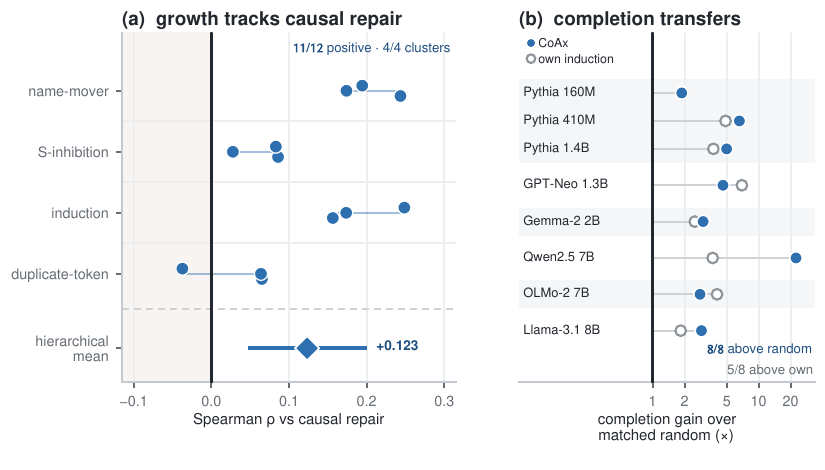}
\vspace{-3mm}
\caption{\textbf{Conditional repair generalizes across mechanisms and architectures.}
\textbf{(a)} Signed growth has positive rank agreement with mechanism-matched held-out causal repair in $11/12$ instances and all four mechanism clusters; the bottom row shows the hierarchical mean and nested $95\%$ CI.
\textbf{(b)} \coax-selected completions outperform matched random completions on all eight non-GPT-2 models spanning six architecture families; hollow circles show each model's next induction heads as a stronger role-matched control.}
\label{fig:regime}
\end{figure}

\paragraph{Conditional circuit completion transfers beyond GPT-2.}
We finally test set-level transfer on eight non-GPT-2 models spanning six architecture families.
Under disjoint detection, calibration, and evaluation splits with equal selector budgets, adding \coax-selected induction heads outperforms a size-matched random completion on all eight models (Figure~\ref{fig:regime}b).
A stronger role-matched control, each model's own next induction heads, is also load-bearing; \coax gives the larger completion on five of eight models. Full held-out comparisons against single ablation, co-activation, AtP, and AtP$^\star$ are in Appendix~\ref{app:gen-group}.
Thus conditional circuit completion transfers beyond GPT-2-small across diverse induction circuits.
Taken together, cross-mechanism repair alignment and cross-model circuit completion show that primary-set removal is more than a diagnostic perturbation: it exposes how causal responsibility reorganizes after failure and provides the information needed to complete the circuit account that the intact model alone leaves hidden.

\section{Related Work}
\label{sec:related}

\paragraph{Circuit discovery and completeness.}
Mechanistic interpretability identifies circuits by intervening on components or edges and measuring how model behavior changes.
Automated Circuit Discovery (ACDC) automates intervention-based pruning~\citep{conmy2023acdc}; edge attribution patching and its integrated-gradient variant estimate edge effects efficiently~\citep{syed2023attribution,hanna2024faith}; and AtP$^\star$ improves component localization with gradient-based corrections~\citep{kramar2024atp}.
Subsequent work refines circuit representations, search procedures, or pruning objectives~\citep{marks2024sparse,ferrando2024information,hsu2025cdt,bhaskar2024edge}.
These methods primarily seek the computation supporting a reference behavior.
Complementary work on circuit completeness shows that redundant, OR-like paths can be omitted by standard discovery and combines noising and denoising interventions to recover them~\citep{chen2025completeness}.
\coax focuses on completeness relative to a specified failure state: given a supplied primary set, it asks which remaining components become causally important after that set is removed.

\paragraph{Self-repair and compensatory circuitry.}
Self-repair makes this state dependence concrete.
The Hydra effect~\citep{mcgrath2023hydra}, documented IOI backup name-movers~\citep{wang2022ioi}, copy-suppression~\citep{mcdougall2023copysuppression}, and LayerNorm and anti-erasure mechanisms~\citep{rushing2024selfrepair} show several ways transformer computation can reorganize after intervention.
Methods such as AtP$^\star$ GradDrop and GIM address attribution distortions caused by this reorganization, while joint-ablation analyses expose redundancy hidden by isolated effects~\citep{kramar2024atp,edin2025gim,ye2026hiddenheroes}.
\coax complements these analyses by localizing the compensating candidates relative to a supplied primary set, then testing whether they causally carry repair and complete the circuit.

\paragraph{Conditional and interaction-aware effects.}
State-dependent importance is inherently relational: a candidate's effect can depend on what else has been removed.
Neuron Shapley averages marginal effects across many coalitions~\citep{ghorbani2020neuronshapley}, while interventional redundancy--synergy decompositions characterize dependence across removal contexts~\citep{aghilar2026scalarlimit}.
Closest to our theory, causal-mediation analysis shows that component-wise patching estimands contain pairwise and higher-order interaction effects when one component's causal role depends on others~\citep{mediators2026curse}.
The M\"obius decomposition itself is standard~\citep{rota1964foundations}; our contribution is to make the supplied failure set part of the query and use the resulting all-order aggregate as the estimand for conditional completion.
Input-side co-activation can expose components that fire together but does not condition on a specified failure~\citep{closure2026}.
Signed growth uses the intervention-induced increase in ablation-effect magnitude to target recruitment after the supplied failure.

\section{Conclusion}
Self-repair makes causal importance conditional on what remains available, so an intact-state circuit can cease to explain the model once that circuit is perturbed.
\coax turns this difficulty into a probe: removing a primary circuit reveals which previously weak components become causally important and allows the missing response to be incorporated back into the circuit account.
Across documented backups, independent repair targets, and multiple architectures, this conditional view exposes intervention-dependent circuitry that intact-state localization alone cannot characterize.
A causal explanation of a self-repairing model is therefore not a single intact-state circuit, but an account of how causal responsibility reorganizes when that circuit is disrupted.

\clearpage
\paragraph*{Reproducibility statement.}
The paper is designed for direct reproduction: Section~\ref{sec:method} and Appendix~\ref{app:method} specify the estimand, algorithm, implementation, and complexity; Appendix~\ref{app:proofs} contains complete proofs; Table~\ref{tab:hparams} consolidates protocol constants and random seeds; and Appendix~\ref{app:setup} specifies the models, data, metrics, and baselines.
Appendix~\ref{app:repro} reports compute and reproduction details.
All model weights are frozen and no training is performed; the core GPT-2-small IOI experiments can also be reproduced on CPU.

\paragraph*{Ethics statement.}
This work performs inference-only analysis of publicly released pretrained language models and uses no human subjects, private data, or personally identifying information.
Model licenses, data sources, and broader-impact considerations are documented in Appendix~\ref{app:limitations}.

\bibliography{coablation}
\bibliographystyle{conference_preprint}

\clearpage
\appendix

\section*{Appendix Overview}

The appendix is organized into six parts, following the technical and empirical progression of the paper.
Parts~A--B specify the method and shared experimental protocol; Parts~C--E provide the detailed evidence behind Sections~\ref{sec:discovery}--\ref{sec:panel}; and Part~F records scope, ethics, and reproducibility information.

\paragraph{Part A: Method Details.}
Appendix~\ref{app:method} provides the complete technical specification of \coax.
\vspace{-1mm}

\noindent\textbf{Appendix~\ref{app:geometry}: Notation and output geometry.}
Formal notation, terminology, Fisher geometry, and the finite-ablation interpretation.\\
\noindent\textbf{Appendix~\ref{app:proofs}: Proofs.}
Complete proofs of Propositions~\ref{prop:blind} and~\ref{prop:exact}, together with the signed-growth decomposition.\\
\noindent\textbf{Appendix~\ref{app:algorithm}: Algorithm, cost, and implementation.}
The full algorithm, pass complexity, measured efficiency, and implementation details.\\
\noindent\textbf{Appendix~\ref{app:hyperparams}: Hyperparameters and robustness.}
Protocol constants, ablation-value robustness, calibration efficiency, output-geometry analysis, and task-alignment diagnostics.

\vspace{1mm}
\paragraph{Part B: Experimental Setup.}
Appendix~\ref{app:setup} defines the models, circuit roles and supplied primary sets, evaluation data, metrics, and intact-state, conditional, and structural baselines used throughout the experiments.

\vspace{1mm}
\paragraph{Part C: Backup Recovery and Mechanistic Validation.}
Appendix~\ref{app:discovery} expands Section~\ref{sec:discovery}.
\vspace{-1mm}

\noindent\textbf{Appendix~\ref{app:sig-group}: Backup recovery and statistical validation.}
Recovery across seeds, significance and imbalance-sensitive metrics, historical-readout and layer-depth controls, and a primary-removal-free reference construction.\\
\noindent\textbf{Appendix~\ref{app:specificity}: Recovery specific to the removed primary set.}
Alternative and geometry-matched primary sets, contamination, and LayerNorm-freezing controls.\\
\noindent\textbf{Appendix~\ref{app:real-group}: Causal validation of the recovered branch.}
Circuit context, graded recruitment and dose response, selector-specific freezing, structural corroboration, and a concrete backup head.\\
\noindent\textbf{Appendix~\ref{app:keys-group}: What drives conditional growth.}
The intact-to-removed contrast, output geometry, controlled dormancy and task-alignment tests, and empirical non-additivity.\\
\noindent\textbf{Appendix~\ref{app:robust-group}: Robustness.}
Primary-set size, prompt-template robustness, and an implementation check on other additive units.

\vspace{1mm}
\paragraph{Part D: Conditional Circuit Completion and Knockout.}
Appendix~\ref{app:attribution} expands Section~\ref{sec:apps}.
\vspace{-1mm}

\noindent\textbf{Appendix~\ref{app:completeall}: Circuit-completion controls.}
Raw headline completion, matched controls, and prompt-level uncertainty.\\
\noindent\textbf{Appendix~\ref{app:oracleko}: Knockout against the documented backup intervention.}
Counterfactual matching in IOI margin and output distribution, together with collateral-change and paired-uncertainty controls.

\vspace{1mm}
\paragraph{Part E: Conditional Repair Beyond Documented Backups.}
Appendix~\ref{app:panel-full} expands the two generalization tests of Section~\ref{sec:panel}.
\vspace{-1mm}

\noindent\textbf{Appendix~\ref{app:panel}: Repair prediction across mechanisms.}
The twelve held-out instances under mechanism-matched causal targets and hierarchical inference.\\
\noindent\textbf{Appendix~\ref{app:gen-group}: Conditional circuit completion across models.}
Held-out, equal-budget comparisons underlying Figure~\ref{fig:regime}b.

\vspace{1mm}
\paragraph{Part F: Limitations, Broader Impact, Ethics, and Reproducibility.}
Appendix~\ref{app:limitations} records the empirical scope and granularity of the study, broader-impact and ethics considerations, model and data licenses, compute, and reproduction details.

\clearpage
\section{Method Details}\label{app:method}

\subsection{Notation and Output Geometry}
\label{app:geometry}

\paragraph{Notation and conventions.}
\label{app:terms}
Table~\ref{tab:notation} summarizes the symbols for the theoretical development.

\begin{table}[ht]\centering\small
\centering
\small
\renewcommand{\arraystretch}{1.08}
\setlength{\tabcolsep}{5pt}
\caption{Core notation for the conditional-completion analysis.}
\label{tab:notation}
\begin{tabular}{@{}p{0.13\linewidth}p{0.70\linewidth}@{}}
\toprule
\rowcolor{hdr}
\textbf{Symbol} & \textbf{Meaning} \\
\midrule

\rowcolor{hdr!55}
\multicolumn{2}{c}{\textbf{Problem objects}} \\

$\Sset$ & supplied primary set whose removal defines the conditional intervention \\
$\Uset$ & candidate units, with $\Uset\cap\Sset=\emptyset$ \\
$u$ & a candidate unit in $\Uset$ \\

\addlinespace[2pt]
\rowcolor{hdr!55}
\multicolumn{2}{c}{\textbf{Sets used in the interaction decomposition}} \\

$M$ & generic ablated set, with logits $z_M$ and joint effect $\dz_M=z_0-z_M$ \\
$T\subseteq M$ & interaction set contributing to the decomposition of $\dz_M$ \\
$A\subseteq T$ & dummy subset index in the M\"obius transform defining $I_T$ \\
$R\subseteq\Sset$ & primary-set subset in Proposition~2, with $T=R\cup\{u\}$ \\

\addlinespace[2pt]
\rowcolor{hdr!55}
\multicolumn{2}{c}{\textbf{Effects and scores}} \\

$\dz_u$ & ablation effect of $u$ with $\Sset$ present \\
$\dz_{u\mid\Sset}$ & ablation effect of $u$ after $\Sset$ is removed \\
$\Delta_u(\Sset)$ & effect change $\dz_{u\mid\Sset}-\dz_u$ induced by primary-set removal \\
$I_T$ & M\"obius interaction associated with the unit set $T$ \\
$\energy(v)$ & magnitude of a logit perturbation in the fixed clean output geometry \\
$\comp_u(\Sset)$ & \coax conditional-growth score \\
\bottomrule
\end{tabular}
\end{table}

\paragraph{Output geometry.}
We first justify the output measure in Eq.~\ref{eq:energy}.
Fix an evaluated input--position pair $(x,t)$ and suppress this dependence below;
recall that $p_0=\mathrm{softmax}(z_0)$ is the clean output distribution.
For a logit perturbation $v$, the centered term
$v-\langle p_0,v\rangle\mathbf{1}$, where $\mathbf{1}$ is the all-ones vector,
removes the shared-logit-shift direction to which softmax is invariant, while the
subsequent $\sqrt{p_0}$ weighting yields the transformed perturbation $\widetilde v$.
As shown below, Euclidean inner products between such transformed perturbations
coincide exactly with Fisher inner products at the clean output distribution $p_0$.

For vocabulary-logit perturbations $v$ and $w$, let
$\bar v=\langle p_0,v\rangle$ and $\bar w=\langle p_0,w\rangle$
denote their $p_0$-weighted means.
Using the transformation in Eq.~\ref{eq:energy},
\begin{equation}
\begin{aligned}
\langle\widetilde v,\widetilde w\rangle
&=
\sum_i (p_0)_i(v_i-\bar v)(w_i-\bar w) \\
&=
\sum_i (p_0)_i v_iw_i-\bar v\,\bar w \\
&=
v^\top
\left(\mathrm{diag}(p_0)-p_0p_0^\top\right)w
=
v^\top Fw .
\end{aligned}
\label{eq:app_fisher_identity}
\end{equation}
Thus Eq.~\ref{eq:energy} is exactly the squared Fisher-metric magnitude of the
perturbation, averaged over the evaluated input--position pairs:
\begin{equation}
\energy(v)
=
\mathbb{E}_{x,t}\!\left[v^\top Fv\right].
\label{eq:app_energy_fisher}
\end{equation}
Because $F$ is positive semidefinite, $\energy(v)\geq 0$.
Moreover, adding the same constant to every logit leaves both the softmax
distribution and $\energy(v)$ unchanged, so the measure discards the shared-shift
direction that is invisible to softmax.

This Fisher metric also has a direct distributional interpretation.
For a scalar path parameter $\alpha$, define the straight-line logit path
$z_\alpha=z_0-\alpha v$ and the corresponding scalar function
\begin{equation}
\phi(\alpha)
:=
\mathrm{KL}\!\left(
p_0\,\middle\|\,\mathrm{softmax}(z_0-\alpha v)
\right).
\label{eq:app_phi}
\end{equation}
Since
$\nabla_z\mathrm{KL}(p_0\,\|\,\mathrm{softmax}(z))
=\mathrm{softmax}(z)-p_0$,
the clean point satisfies $\phi(0)=0$ and $\phi'(0)=0$.
Moreover, the Hessian at $z_0$ is the categorical Fisher matrix
$F=\mathrm{diag}(p_0)-p_0p_0^\top$, so
\begin{equation}
\phi''(0)
=
v^\top Fv .
\label{eq:app_phi_second}
\end{equation}
A Taylor expansion around $\alpha=0$ therefore gives
\begin{equation}
\mathrm{KL}\!\left(
p_0\,\middle\|\,\mathrm{softmax}(z_0-\alpha v)
\right)
=
\frac{\alpha^2}{2}\,v^\top Fv
+
O(\alpha^3\|v\|_2^3),
\qquad \alpha\to0 .
\label{eq:app_fisher_kl}
\end{equation}
Hence $v^\top Fv$ characterizes the local quadratic sensitivity of the output
distribution along the logit direction $v$.
This gives the interpretation used throughout the paper: $\energy(v)$ measures
the magnitude of an observed logit displacement in the fixed clean output geometry.

\paragraph{Interpretation for finite ablations.}
\label{app:softmask}
The local KL expansion above motivates the geometry of the output measure, while
\coax itself is applied to finite ablations.
For each intervention, we directly observe the resulting logit displacement and
measure its magnitude using the Fisher matrix $F$ at the clean output distribution
as a fixed reference geometry.

Eq.~\ref{eq:app_fisher_kl} therefore describes only the local behavior of the
straight-line path $z_\alpha=z_0-\alpha v$ in logit space.
A full internal ablation need not follow this path: for example, if $h_u$ denotes
the residual-stream contribution of component $u$, gradually scaling it as
$h_u\leftarrow(1-\alpha)h_u$ is followed by nonlinear transformer computation,
so the resulting logits need not equal $z_0-\alpha\dz_u$.
Accordingly, the method does not require the full-ablation KL to equal its
quadratic approximation, nor does it assume a linear internal ablation trajectory.
The Fisher expansion provides a local geometric interpretation of the metric used
to compare the finite displacements actually produced by intervention.

\subsection{Proofs}\label{app:proofs}
We prove Proposition~\ref{prop:blind} by an explicit non-identifiability
construction and Proposition~\ref{prop:exact} by M\"obius inversion, and then
derive the score decomposition in Eq.~\ref{eq:score_shift}.
Except where $\energy$ is applied, the identities below hold
pointwise for each evaluated input--position pair $(x,t)$.

\subsubsection{Proof of Proposition~\ref{prop:blind}}

\begin{proof}
Recall that $\mathcal{I}_u$ denotes the per-unit information recorded for
candidate $u$ on the original model: its forward activations, output derivatives,
and single-ablation effect $\dz_u$.
We construct two candidates with identical such records but different causal
roles after the primary set is removed.

It suffices to consider an input distribution supported on one evaluated
input--position pair.
Choose a nonconstant logit direction $w$.
Since the clean softmax distribution $p_0$ has strictly positive entries,
Eq.~\ref{eq:app_energy_fisher} gives
\begin{equation}
\energy(w)=w^\top Fw>0 .
\label{eq:app_blind_positive}
\end{equation}

Let $q$ denote a scalar state supplied by the primary path: $q=1$ when
$\Sset$ is present and $q=0$ when it is removed.
Let $a_b$ and $a_n$ denote the activations of a prospective backup $b$ and an
irrelevant candidate $n$, respectively.
Take $a_b=a_n=1$ on the original model, with ablation setting the corresponding
activation to zero.
Consider the downstream logit readout
\begin{equation}
z(q,a_b,a_n)
=
z_{\mathrm{base}}
+
\bigl[q+(1-q)a_b\bigr]\,w ,
\label{eq:app_blind_readout}
\end{equation}
where $z_{\mathrm{base}}$ collects all remaining logit contributions.
The candidate $n$ does not enter this readout.

When $\Sset$ is present, $q=1$, so the coefficient of $w$ is one regardless
of $a_b$.
Consequently, ablating either candidate leaves the clean logits unchanged:
\begin{equation}
z_{\{b\}}=z_{\{n\}}=z_0=z_{\mathrm{base}}+w,
\qquad
\dz_b=\dz_n=0 .
\label{eq:app_blind_intact}
\end{equation}
Moreover, at the intact state,
\begin{equation}
\frac{\partial z}{\partial a_b}
=
(1-q)w
=
0,
\qquad
\frac{\partial z}{\partial a_n}
=
0 .
\label{eq:app_blind_derivatives}
\end{equation}
Thus $b$ and $n$ have the same forward activation, the same output derivative,
and the same single-ablation effect on the original model.
By Definition~\ref{def:intactscore},
\begin{equation}
\mathcal{I}_b=\mathcal{I}_n,
\qquad
\energy(\dz_b)=\energy(\dz_n)=0 .
\label{eq:app_blind_same_record}
\end{equation}
Hence every intact-state score assigns the two candidates the same value.

Now remove the primary set, so $q=0$.
With $b$ present, its activation supplies exactly the missing logit direction
$w$:
\begin{equation}
z_\Sset=z_{\mathrm{base}}+w=z_0 .
\label{eq:app_blind_repair}
\end{equation}
Thus $b$ perfectly repairs the primary-path contribution at the output.
If $b$ is also ablated, that contribution disappears, whereas ablating $n$
still changes nothing:
\begin{equation}
z_{\Sset\cup\{b\}}=z_{\mathrm{base}},
\qquad
z_{\Sset\cup\{n\}}=z_\Sset .
\label{eq:app_blind_counterfactuals}
\end{equation}
Therefore the conditional ablation effects are
\begin{equation}
\dz_{b\mid\Sset}
=
z_\Sset-z_{\Sset\cup\{b\}}
=
w,
\qquad
\dz_{n\mid\Sset}
=
z_\Sset-z_{\Sset\cup\{n\}}
=
0 .
\label{eq:app_blind_conditional}
\end{equation}
Using Eq.~\ref{eq:app_blind_positive},
\begin{equation}
\energy(\dz_{b\mid\Sset})>0,
\qquad
\energy(\dz_{n\mid\Sset})=0 .
\label{eq:app_blind_conditional_energy}
\end{equation}
Finally, because both intact effects vanish, Eq.~\ref{eq:comp} gives
\begin{equation}
\comp_b(\Sset)
=
\energy(w)
>0,
\qquad
\comp_n(\Sset)
=
0 .
\label{eq:app_blind_coax}
\end{equation}
Thus the two candidates are indistinguishable from their intact-state records,
while conditional co-ablation separates the dormant backup $b$ from the
irrelevant candidate $n$.
\end{proof}

\paragraph{Interpretation.}
Proposition~\ref{prop:blind} establishes a non-identifiability result for the
intact-state information class in Definition~\ref{def:intactscore}.
The same per-unit evidence on the original model can correspond to different
causal roles after the primary set is removed: in the construction above, $b$
exactly replaces the lost primary contribution $w$, while $n$ remains irrelevant.
The missing distinction is therefore genuinely conditional on primary-set removal.
Perfect dormancy gives the sharpest case; the experiments examine whether this
distinction remains recoverable when real backups retain some intact-state signal.

\subsubsection{Proof of Proposition~\ref{prop:exact}}

\begin{proof}
Fix a candidate $u\in\Uset$, so $u\notin\Sset$.
Recall that for any ablated set $M$,
$\dz_M=z_0-z_M$, with $\dz_\emptyset=0$.
For the vector-valued set function $M\mapsto\dz_M$, the M\"obius interaction
associated with each nonempty set $T$ is
\begin{equation}
I_T
=
\sum_{A\subseteq T}
(-1)^{|T|-|A|}\dz_A ,
\label{eq:app_mobius_transform}
\end{equation}
where $A$ is a dummy subset of $T$.
M\"obius inversion on the subset lattice~\citep{rota1964foundations} applies
coordinatewise to the logit vector and gives, for every ablated set $M$,
\begin{equation}
\dz_M
=
\sum_{\emptyset\neq T\subseteq M} I_T .
\label{eq:app_mobius_inversion}
\end{equation}

The conditional ablation effect can be written as the difference between two
joint ablation effects:
\begin{equation}
\dz_{u\mid\Sset}
=
z_\Sset-z_{\Sset\cup\{u\}}
=
\dz_{\Sset\cup\{u\}}-\dz_\Sset .
\label{eq:app_conditional_difference}
\end{equation}
Applying Eq.~\ref{eq:app_mobius_inversion} to the two terms yields
\begin{equation}
\begin{aligned}
\dz_{u\mid\Sset}
&=
\sum_{\emptyset\neq T\subseteq\Sset\cup\{u\}} I_T
-
\sum_{\emptyset\neq T\subseteq\Sset} I_T \\
&=
\sum_{\substack{T\subseteq\Sset\cup\{u\}\\u\in T}} I_T .
\end{aligned}
\label{eq:app_exact_cancel}
\end{equation}
Indeed, every interaction whose set $T$ does not contain $u$ is a subset of
$\Sset$ and therefore appears in both sums, so these terms cancel exactly.

Because $u\notin\Sset$, every surviving set containing $u$ can be written
uniquely as
$T=R\cup\{u\}$ for a subset $R\subseteq\Sset$.
Re-indexing the remaining terms therefore gives
\begin{equation}
\dz_{u\mid\Sset}
=
\sum_{R\subseteq\Sset} I_{R\cup\{u\}} .
\label{eq:app_exact_reindex}
\end{equation}
The term $R=\emptyset$ is
$I_{\{u\}}=\dz_u$.
More generally, an index set $R$ of size $|R|$ contributes an interaction of
order $|R|+1$: $|R|=1$ gives the pairwise candidate--primary interactions,
while $|R|\geq2$ gives all higher orders.

Subtracting the first-order term gives the effect change induced by
primary-set removal:
\begin{equation}
\Delta_u(\Sset)
:=
\dz_{u\mid\Sset}-\dz_u
=
\sum_{\emptyset\neq R\subseteq\Sset}
I_{R\cup\{u\}} .
\label{eq:app_delta_exact}
\end{equation}
Thus $\Delta_u(\Sset)$ is exactly the aggregate of every interaction order
that contains the candidate $u$ and at least one unit from the removed primary
set.
No approximation, pairwise truncation, or assumption of linear model
responses is involved; higher-order nonlinear joint effects are represented
by the corresponding $I_T$ terms.
\end{proof}

\subsubsection{Score decomposition}

We now connect the all-order effect change in Proposition~\ref{prop:exact} to the scalar \coax score. By definition,
\begin{equation}
\dz_{u\mid\Sset}=\dz_u+\Delta_u(\Sset),
\label{eq:app_effect_change_identity}
\end{equation}
where Eq.~\ref{eq:app_delta_exact} shows that $\Delta_u(\Sset)$ is exactly the aggregate of all interaction orders linking $u$ to the removed primary set.

For each evaluated input--position pair $(x,t)$, the clean distribution $p_0$ is fixed, so the transformation $v\mapsto\widetilde v$ in Eq.~\ref{eq:energy} is linear. Hence
\begin{equation}
\widetilde{\dz}_{u\mid\Sset}=\widetilde{\dz}_u+\widetilde{\Delta}_u(\Sset),
\label{eq:app_tilde_delta}
\end{equation}
where $\widetilde{\Delta}_u(\Sset)$ is the same centered, Fisher-weighted transformation applied to the effect change.
Substituting into $\comp_u(\Sset)=\energy(\dz_{u\mid\Sset})-\energy(\dz_u)$ gives
\begin{equation}
\begin{aligned}
\comp_u(\Sset)
&=\mathbb{E}_{x,t}\left[\left\|\widetilde{\dz}_u+\widetilde{\Delta}_u(\Sset)\right\|_2^2-\left\|\widetilde{\dz}_u\right\|_2^2\right]\\
&=\mathbb{E}_{x,t}\left[2\left\langle\widetilde{\dz}_u,\widetilde{\Delta}_u(\Sset)\right\rangle+\left\|\widetilde{\Delta}_u(\Sset)\right\|_2^2\right].
\end{aligned}
\label{eq:app_score_decomposition}
\end{equation}
This is Eq.~\ref{eq:score_shift}. The squared-change term measures the magnitude of the \emph{aggregate} interaction-induced effect change; because $\Delta_u(\Sset)$ is a vector sum, different interaction orders may reinforce or cancel before this norm is taken. For a perfectly dormant unit, $\energy(\dz_u)=0$, so the intact-state cross term vanishes and
\begin{equation}
\comp_u(\Sset)=\energy\!\left(\Delta_u(\Sset)\right).
\label{eq:app_dormant_growth}
\end{equation}
Thus the score directly measures the newly exposed conditional effect in the sharp dormant case, while retaining the signed change in ablation-effect magnitude away from that limit.

\subsection{Algorithm, Cost, and Implementation}
\label{app:algorithm}

\subsubsection{Algorithm and pass complexity}

Algorithm~\ref{alg:cond} gives the direct finite-ablation implementation of the \coax score in Eq.~\ref{eq:comp}.
Let $\mathcal{X}$ denote the calibration set of evaluated input--position pairs
$(x,t)$; all expectations below are empirical averages over $\mathcal{X}$.
The clean output distribution $p_0(x,t)$ is computed once and then held fixed
across all interventions.

\begin{algorithm}[ht]
\caption{\coax conditional co-ablation}
\label{alg:cond}
\begin{algorithmic}[1]
\STATE \textbf{Input:} frozen model; calibration set $\mathcal{X}$ of input--position pairs $(x,t)$; primary set $\Sset$; candidate set $\Uset$ with $\Uset\cap\Sset=\emptyset$
\STATE \textbf{Output:} conditional-growth scores $\comp_u(\Sset)$ for all $u\in\Uset$
\STATE compute clean logits $z_0(x,t)$ and
$p_0(x,t)\leftarrow\mathrm{softmax}(z_0(x,t))$
for all $(x,t)\in\mathcal{X}$
\FOR{$u\in\Uset$}
  \STATE ablate $u$ and compute
  $\dz_u\leftarrow z_0-z_{\{u\}}$
  \STATE compute
  $\widetilde{\dz}_u
  \leftarrow
  \sqrt{p_0}\odot
  \bigl(\dz_u-\langle p_0,\dz_u\rangle\mathbf 1\bigr)$
  \STATE compute and store
  $\energy(\dz_u)
  \leftarrow
  \mathbb{E}_{(x,t)\in\mathcal{X}}
  \|\widetilde{\dz}_u\|_2^2$
\ENDFOR
\STATE ablate $\Sset$ jointly and compute $z_\Sset$
\COMMENT{shared conditional reference}
\FOR{$u\in\Uset$}
  \STATE ablate $\Sset\cup\{u\}$ and compute
  $\dz_{u\mid\Sset}
  \leftarrow
  z_\Sset-z_{\Sset\cup\{u\}}$
  \STATE compute
  $\widetilde{\dz}_{u\mid\Sset}
  \leftarrow
  \sqrt{p_0}\odot
  \bigl(
  \dz_{u\mid\Sset}
  -
  \langle p_0,\dz_{u\mid\Sset}\rangle\mathbf 1
  \bigr)$
  \STATE
  $\comp_u(\Sset)
  \leftarrow
  \mathbb{E}_{(x,t)\in\mathcal{X}}
  \|\widetilde{\dz}_{u\mid\Sset}\|_2^2
  -
  \energy(\dz_u)$
  \COMMENT{conditional growth}
\ENDFOR
\STATE \textbf{return} $\comp_u(\Sset)$ for ranking
\end{algorithmic}
\end{algorithm}

\paragraph{Pass complexity.}
The computation mirrors the two states compared by \coax.
The clean logits $z_0$ and the primary-ablated logits $z_\Sset$ are shared
across all candidates, while each $u\in\Uset$ requires two candidate-specific
evaluations: $z_{\{u\}}$ for its intact ablation effect and
$z_{\Sset\cup\{u\}}$ for its conditional ablation effect.
Thus scoring $|\Uset|$ candidates requires $2|\Uset|+2$ forward evaluations
over the calibration set when both shared references are counted, or
$2|\Uset|+1$ additional evaluations when the clean reference is already
available.
In either convention, the pass complexity is $O(|\Uset|)$, with no backward
pass or task gradient.

\paragraph{Fixed output geometry.}
The same clean distribution $p_0$ is used for both the intact and conditional effects, so their magnitudes are compared in a common output geometry.

\paragraph{Direct access to the interaction aggregate.}
Proposition~\ref{prop:exact} identifies the effect change
$\Delta_u(\Sset)$ as the exact aggregate of all interaction orders linking
candidate $u$ to the removed primary set.
\coax does not recover those constituent terms individually: doing so would, in general, require subset evaluations, with even a pairwise truncation costing $O(|\Uset||\Sset|)$ and omitting higher orders.
Instead, the two-state scan measures the complete vector aggregate $\Delta_u(\Sset)$ directly in $O(|\Uset|)$ evaluations.

\paragraph{Approximation and reuse.}
For the optional top-$r$ large-vocabulary approximation, the top-$r$
vocabulary support is selected once from the clean distribution for each
evaluated input--position pair and then held fixed across all interventions.
The clean probabilities are renormalized on that support before applying
Eq.~\ref{eq:energy}; its effect on backup recovery is reported in
Appendix~\ref{app:components}.
If circuit completion is repeated after enlarging $\Sset$, the intact effects
$\dz_u$, their transformed representations, and $\energy(\dz_u)$ can be reused
for candidates that remain outside the primary set.
Only the new shared conditional reference $z_\Sset$ and the corresponding
conditional candidate evaluations $z_{\Sset\cup\{u\}}$ need to be recomputed.

\subsubsection{Cost}

The $O(|\Uset|)$ pass complexity translates to modest wall-clock overhead in
practice.
Table~\ref{tab:scal} reports measured runtime and peak GPU memory for one
intact single-unit scan and for the complete two-state \coax scoring pass, each using $48$ calibration prompts.
Across architectures from GPT-2-small to Llama-3.1-8B, scoring takes seconds to
minutes on a single GPU and requires no backward pass.
Runtime reflects both the cost of a model forward pass and the number of scored
units, while peak memory is dominated by the frozen model.

\begin{table}[ht]
\centering
\small
\setlength{\tabcolsep}{4.2pt}
\caption{Measured cost of \coax on a single GPU with $48$ calibration prompts.
\emph{Single} denotes one intact candidate scan; \emph{\coax total} denotes the full two-state computation, including the intact and primary-conditional candidate scans.
Both are linear in the number of scored units.
\emph{Units} counts the scored structured units, one per attention head
($\text{layers}\times\text{query heads}$; on grouped-query models the query heads, not the
key--value groups).
Peak memory includes the frozen model.}
\label{tab:scal}
\begin{tabular}{lccccc}
\toprule
\rowcolor{hdr}
model & params & units & single (s) & \coax total (s) & peak GB \\
\midrule
GPT-2-small   & 124M & 144 & 2.9 & 5.7 & 0.8 \\
GPT-2-medium  & 355M & 384 & 19  & 39  & 1.7 \\
Pythia-410m   & 410M & 384 & 18  & 37  & 1.9 \\
GPT-Neo-1.3b  & 1.3B & 384 & 64  & 128 & 5.6 \\
Gemma-2-2b    & 2B   & 208 & 18  & 35  & 6.0 \\
\addlinespace[2pt]
OLMo-2-7B     & 7B   & 1024 & 182 & 368 & 14.9 \\
Llama-3.1-8B  & 8B   & 1024 & 180 & 360 & 16.4 \\
\bottomrule
\end{tabular}
\end{table}

\subsubsection{Efficiency}\label{app:efficiency}

\coax uses two forward-only candidate scans, and each scan has a distinct role:
the conditional scan measures candidate effects after the primary set is removed,
while the intact scan removes pre-existing importance to isolate the growth caused
by that intervention.
This two-state design provides three complementary efficiency advantages.

\paragraph{One additional scan buys the conditional contrast.}
The closest matched control is conditional energy alone:
it uses the same supplied primary set, prompts, finite-ablation machinery, and
Fisher geometry as \coax, but ranks candidates only by
$\energy(\dz_{u\mid\Sset})$.
It requires a single conditional candidate scan and reaches
$0.758\pm0.004$ backup AUC.
Adding the intact scan enables the subtraction in Eq.~\ref{eq:comp}, raising
recovery to $0.941\pm0.004$.
Thus the additional $O(|\Uset|)$ scan is spent specifically on separating
intervention-induced growth from importance that was already present before
the intervention.

\paragraph{All-order information without interaction enumeration.}
By Proposition~\ref{prop:exact}, the effect change
$\Delta_u(\Sset)$ obtained from these two states is exactly the aggregate of
all interaction orders linking candidate $u$ to the removed primary set.
Recovering those orders individually is exponential in $|\Sset|$.
Even the natural pairwise truncation requires
$O(|\Uset||\Sset|)$ candidate--primary joint evaluations and omits every
order above two, whereas \coax obtains the complete aggregate in
$O(|\Uset|)$ forward evaluations.
Conditioning therefore obtains the complete effect-change aggregate directly,
without decomposing it into constituent interaction orders.

\paragraph{Forward-only and data-efficient.}
Unlike the gradient baselines, \coax requires neither a backward pass nor a
task gradient; all candidate interventions are independent and can be batched
or parallelized.
On GPT-2-small, the strongest intact-state gradient baseline reaches
$0.665\pm0.009$ backup AUC, compared with $0.941\pm0.004$ for \coax.
The calibration requirement is also small: $16$ unlabeled prompts already
reach $0.931$ AUC, $32$ reach $0.937$, and the $96$-prompt main setting reaches
$0.941$ (Figure~\ref{fig:method_analysis}a).

\begin{table}[ht]
\centering
\small
\setlength{\tabcolsep}{4.5pt}
\caption{
Compute--recovery trade-off on GPT-2-small IOI.
The three supplied primary heads are excluded from the backup ranking, leaving
$|\Uset|=141$ candidates.
Forward/backward counts omit the shared clean reference.
$\dagger$ denotes methods given the same supplied primary set as \coax.
}
\label{tab:efficiency}
\begin{tabular}{lcccc}
\toprule
\rowcolor{hdr}
method & forwards & backwards & task grad. & backup AUC \\
\midrule

\rowcolor{hdr!55}
\multicolumn{5}{c}{\textbf{Intact-state scores}} \\

single ablation
& $|\Uset|$ & $0$ & --- & $0.603\pm.007$ \\

AtP
& $1$ & $1$ & \checkmark & $0.650\pm.010$ \\

activation-space EAP-IG
& $5L$ & $5L$ & \checkmark & $0.581\pm.027$ \\

AtP$^\star$ GradDrop
& $L$ & $L$ & \checkmark & $0.665\pm.009$ \\

\addlinespace[2pt]
\rowcolor{hdr!55}
\multicolumn{5}{c}{\textbf{Given the same primary set}} \\

conditional energy only$^\dagger$
& $|\Uset|+1$ & $0$ & --- & $0.758\pm.004$ \\

pairwise-truncated growth$^\dagger$
& $|\Uset|+|\Sset|+|\Uset||\Sset|$
& $0$ & --- & $0.927\pm.001$ \\

\rowcolor{ourrow}
\coax$^\dagger$
& $2|\Uset|+1$ & $0$ & --- & \bestauc{0.941}{.004} \\

\bottomrule
\end{tabular}

\begin{flushleft}
\scriptsize
\textbf{Notes:}
For IOI, $|\Uset|=141$ and $|\Sset|=3$, giving $283$ forward evaluations
for \coax versus $567$ for the pairwise-truncated construction.
The latter retains only
$\Delta_u^{(2)}(\Sset)=\sum_{s\in\Sset}I_{\{u,s\}}$,
whereas Proposition~\ref{prop:exact} shows that \coax obtains the complete
all-order aggregate $\Delta_u(\Sset)$.
\end{flushleft}

\end{table}

Table~\ref{tab:efficiency} summarizes the resulting compute--recovery trade-off.
For the IOI backup task, $|\Uset|=141$ candidates remain after excluding the
three supplied primary heads.
The clean reference, shared by all methods that require it, is omitted from
the pass counts.

\subsubsection{Implementation details}

The following conventions remove implementation ambiguities to make the score precise and support reproduction.

\begin{itemize}

\item \textbf{Aggregation.}
For each prompt, $\energy(\cdot)$ is first averaged over its evaluated token positions and then averaged across prompts, giving every prompt equal weight.
In the main IOI discovery setting there is a single evaluated position---the answer position---as specified in Table~\ref{tab:hparams}.

\item \textbf{Top-$r$ vocabulary support.}
When the optional top-$r$ approximation is used, a support is selected once from the clean distribution $p_0$ for each evaluated input--position pair and
is held fixed for all corresponding ablated passes.
The clean probabilities are renormalized on this support before applying Eq.~\ref{eq:energy}, so the intact and conditional effects are measured in the
same coordinates and clean-reference geometry.

\item \textbf{Conditional intervention.}
All units in the supplied primary set $\Sset$ are ablated jointly, rather than one at a time.
The resulting logits $z_\Sset$ form a single conditional reference shared by all candidates, so every $\dz_{u\mid\Sset}$ is measured relative to the same primary-ablated state.

\item \textbf{Ranking.}
Candidates are ranked directly by the signed growth score $\comp_u(\Sset)$, with no ReLU or post-hoc threshold.
Positive values indicate amplification after primary-set removal, whereas negative values indicate attenuation and therefore naturally fall lower in a
backup ranking.

\item \textbf{Grouped-query attention.}
A unit is one query head, and the intervention overwrites that head's slice of the input to the attention output projection.
Grouped-query attention (GQA) shares keys and values across the query heads of a group, but each query head still writes through its own slice of that projection, so a single query head remains a well-defined component to remove and no query heads are merged into one unit.
The unit counts in Table~\ref{tab:scal} are therefore $\text{layers}\times\text{query heads}$ on every model, including the GQA ones.

\end{itemize}

\subsection{Hyperparameters and Robustness}
\label{app:hyperparams}

\subsubsection{Hyperparameters}

\coax has few tunable quantities; most settings specify the intervention and evaluation protocol rather than learned or task-optimized hyperparameters.
The main IOI discovery setting uses $96$ calibration prompts, reads logits at the \emph{answer position}, and evaluates the \emph{full vocabulary}.
Zero ablation is the default intervention.

The supplied primary set follows the documented circuit definition: three name-mover heads for IOI and four induction heads for the induction experiments.
Robustness to smaller supplied primary sets is reported in Appendix~\ref{app:seedrobust}.
For large-vocabulary models, we optionally use the fixed-support top-$r$ approximation described above; its robustness relative to full-vocabulary scoring is evaluated in Appendix~\ref{app:components}.

\begin{table}[ht]
\centering
\small
\setlength{\tabcolsep}{5pt}
\caption{Main protocol and auxiliary experimental settings.
Task- or model-specific deviations are stated in the corresponding sections.}
\label{tab:hparams}
\begin{tabular}{ll}
\toprule
\rowcolor{hdr}
setting & value \\
\midrule

\rowcolor{hdr!55}
\multicolumn{2}{c}{\textbf{Main discovery protocol}} \\
output position
& answer position (last prompt position) \\
vocabulary
& full ($50{,}257$ for GPT-2-small); optional top-$r$ approximation \\
output geometry
& centered Fisher with fixed clean $p_0$ (Eq.~\ref{eq:energy}) \\
discovery prompts
& $96$ IOI calibration prompts \\
ablation value
& zero; alternatives in Table~\ref{tab:ablrobust} \\
primary set $|\Sset|$
& documented circuit definition ($3$ name-movers; $4$ induction heads) \\
prompt seeds
& $\{1,15,22,8\}$ \\

\addlinespace[2pt]
\rowcolor{hdr!55}
\multicolumn{2}{c}{\textbf{Baseline and unit-granularity settings}} \\
EAP-IG integration steps
& $5$ \\
FFN grouping
& $96$ groups for GPT-2-small FFN analysis \\

\addlinespace[2pt]
\rowcolor{hdr!55}
\multicolumn{2}{c}{\textbf{Auxiliary analyses}} \\
random-control draws
& $20$ (attribution / knockout) \\
permutation shuffles
& $10{,}000$ \\
synthetic trials
& $40$ \\

\bottomrule
\end{tabular}
\end{table}

Table~\ref{tab:hparams} consolidates the main protocol together with the baseline and auxiliary-analysis settings used throughout the experiments.

\subsubsection{Robustness to the ablation value}
\label{app:ablrobust}

Faithfulness estimates can depend on the choice of ablation operator
\citep{miller2024faithfulness}.
We therefore recompute the same \coax ranking under four qualitatively
different replacement rules: zero ablation, mean replacement, empirical
resampling, and moment-matched Gaussian replacement.

For mean replacement, the per-dimension activation mean is estimated from the
head's pre-projection activations on the same $96$ IOI calibration prompts,
pooled over all token positions.
For resample replacement, each candidate is replaced by a single activation
slice drawn uniformly from those same (prompt, position) pairs.
Gaussian replacement uses the corresponding empirical per-dimension mean and
standard deviation estimated from the same pooled activations.
For stochastic replacements, the draw is made once per head before scoring and held fixed across the intact and conditional evaluations, so the two states are compared under the same replacement vector and the contrast is not confounded
by resampling noise.

\begin{table}[ht]
\centering
\small
\setlength{\tabcolsep}{8pt}
\caption{
Robustness to the ablation operator on GPT-2-small IOI.
Values are backup-discovery AUC averaged over the four prompt seeds
$\{1,15,22,8\}$.
}
\label{tab:ablrobust}
\begin{tabular}{lcccc}
\toprule
\rowcolor{hdr}
method & zero & mean & resample & Gaussian \\
\midrule
single ablation
& $0.603\pm.007$
& $0.753\pm.008$
& $0.691\pm.011$
& $0.699\pm.018$ \\
\rowcolor{ourrow}
\coax \scriptsize{(ours)}
& \bestauc{0.941}{.004}
& \bestauc{0.974}{.002}
& \bestauc{0.961}{.008}
& \bestauc{0.963}{.003} \\
\bottomrule
\end{tabular}
\end{table}

As shown in Table~\ref{tab:ablrobust}, \coax achieves $0.941$--$0.974$ across the four operators, while single-ablation ranking remains at
$0.603$--$0.753$.
The conditional ranking therefore remains stable across deterministic,
empirical, and stochastic replacement rules rather than depending on a particular ablation value.

\subsubsection{Method analysis: data efficiency, output geometry, and task alignment}
\label{app:caleff}

Figure~\ref{fig:method_analysis} isolates three design properties of \coax:
its calibration requirement, the contribution of the output geometry, and
its dependence on alignment with the task direction.

\begin{figure}[ht]
\centering
\includegraphics[width=\textwidth]{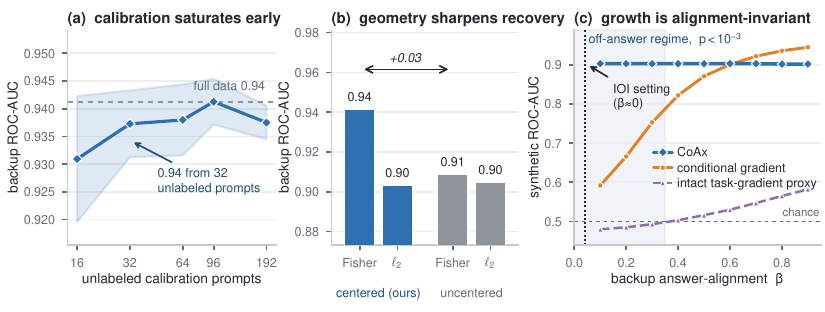}
\caption{\textbf{Diagnostics of the \coax score.}
\textbf{(a)} Backup AUC versus unlabeled calibration-set size
(GPT-2-small IOI, four prompt seeds); recovery is already near the
$96$-prompt result with $16$--$32$ prompts.
\textbf{(b)} Output-geometry ablation under the same conditional contrast:
the full centered Fisher geometry outperforms plain $\ell_2$, centering alone,
and probability weighting alone.
\textbf{(c)} Synthetic backup AUC as answer alignment $\beta$ varies.
Full-distribution \coax remains stable as backups move off the task direction,
whereas the conditional gradient and an intact task-gradient proxy deteriorate;
the latter remains near chance.
The measured IOI setting lies in the low-answer-alignment regime of the sweep.}
\label{fig:method_analysis}
\end{figure}

\paragraph{Data efficiency.}
\coax requires only unlabeled calibration prompts.
On GPT-2-small IOI, $16$ prompts already reach $0.931$ backup AUC,
$32$ reach $0.937$, and the $96$-prompt main setting reaches $0.941$
(Figure~\ref{fig:method_analysis}a).
Thus most of the final recovery is already obtained from a small calibration set.

\paragraph{Output geometry.}
Conditioning provides the main recovery signal (Table~\ref{tab:design});
Figure~\ref{fig:method_analysis}b isolates the additional contribution of the
output geometry.
Using the same conditional contrast, a plain $\ell_2$ magnitude reaches
$0.905$, while the full centered Fisher geometry reaches $0.941$.
Centering alone ($0.903$) and probability weighting alone ($0.909$) do not
recover this gain.
This is consistent with Eq.~\ref{eq:app_fisher_identity}: centering and
$\sqrt{p_0}$ weighting jointly realize the Fisher metric
$F=\mathrm{diag}(p_0)-p_0p_0^\top$.

\paragraph{Full-distribution scoring.}
\coax measures the complete output displacement rather than projecting it
onto the answer direction.
In the synthetic sweep of Figure~\ref{fig:method_analysis}c, $\beta$ controls
how strongly the planted backup effect aligns with that task direction.
The \coax AUC remains essentially unchanged as $\beta$ decreases, whereas the
conditional gradient deteriorates as the backup moves off the answer
direction; intact-state scores remain near chance.
The documented IOI setting also lies in the low-answer-alignment region of this controlled sweep ($\beta\approx0$; Appendix~\ref{app:synth}).
Thus the controlled experiment shows that full-distribution scoring remains stable as a planted conditional effect rotates away from the task direction.

\clearpage
\section{Experimental Setup}
\label{app:setup}

\subsection{Models}
\label{app:models}

We evaluate GPT-2~\citep{radford2019gpt2} (small 124M, medium 355M, large 774M), Pythia~\citep{biderman2023pythia} (160M, 410M, 1.4B), GPT-Neo-1.3B~\citep{black2021gptneo}, Gemma-2~\citep{gemma2024} (2B, 9B), OLMo-2-7B~\citep{olmo2024}, Meta-Llama-3.1-8B-Instruct~\citep{llama3}, and Qwen2.5-7B-Instruct~\citep{qwen2025}.
All checkpoints remain frozen and no method updates model weights.
The cross-model panel uses \texttt{Qwen/Qwen2.5-7B-Instruct} and \texttt{meta-llama/Meta-Llama-3.1-8B-Instruct}.

\subsection{Circuit roles and primary sets}
\label{app:circuitroles}

For GPT-2-small IOI, we use the documented head roles of \citet{wang2022ioi}: name-mover, backup name-mover, S-inhibition, induction, duplicate-token, and previous-token heads.
The main backup experiment takes the three primary name-movers as $\Sset$, the eight documented backups as positives, and all remaining heads as candidates.
For cross-model induction analyses, induction, duplicate-token, and previous-token roles are detected from attention offsets on repeated-random sequences~\citep{olsson2022induction}: previous-token heads attend to $p{-}1$, duplicate-token heads to $p{-}T$, and induction heads to $p{-}T{+}1$ for repeated-sequence length $T$.
These detected roles define primary sets when documented cross-model labels are unavailable.

\subsection{Data}
\label{app:data}

IOI prompts instantiate ``When \{A\} and \{B\} went to the \{place\}, \{B\} gave a \{object\} to'' using the single-token pools of \citet{wang2022ioi}; the indirect object (IO) appears once and the subject (S) twice.
The main discovery protocol uses $96$ sampled fillings, and the structural attention summary in Figure~\ref{fig:attn} is averaged over the same prompts.
Two additional surface templates test prompt-form robustness in Appendix~\ref{app:template}.
Induction experiments use uniformly random token sequences repeated once and score repeated-token log-probability on the second copy.
WikiText-2 windows~\citep{merity2017wikitext} provide unlabeled calibration data for cross-scale geometry analyses.
Protocol constants and prompt seeds are in Table~\ref{tab:hparams}.

\subsection{Metrics}
\label{app:metrics}

For IOI, the behavioral readout is the answer margin $m=z_{\mathrm{IO}}-z_{\mathrm{S}}$.
Backup recovery uses head-level ROC-AUC over the eight documented backups, excluding the supplied primary heads; average precision and precision at small $k$ provide imbalance-sensitive checks.
Section~\ref{sec:panel} reports Spearman rank agreement between each score and held-out freezing-based repair.
Circuit completion uses the incompleteness gap from Section~\ref{sec:apps}; knockout additionally measures IOI-margin distance and answer-position KL to the documented primary-plus-backup intervention, with unrelated-text KL tracking collateral change.

\subsection{Baselines}
\label{app:baselines}

\paragraph{Intact-state scores.}
We compare single-ablation saliency, AtP~\citep{syed2023attribution} (gradient times activation), activation-space EAP-IG~\citep{hanna2024faith} with a zero head-output baseline and five integration steps~\citep{sundararajan2017integrated}, and AtP$^\star$ GradDrop~\citep{kramar2024atp} adapted to the same zero-ablation target.
Following the respective published aggregation conventions, AtP averages per-example attribution magnitudes, whereas activation-space EAP-IG accumulates signed dataset scores and ranks by the magnitude of the aggregate.
For activation-space EAP-IG, we cache clean head outputs and interpolate one attention layer at a time while recomputing all other layers normally.
For GradDrop, each of the $L$ passes detaches one block's residual contribution while preserving the identity residual path; scores take the absolute value per example before aggregation and use the $L-1$ normalization of \citet{kramar2024atp}.
The QK correction instead targets attention saturation and is not applicable to the post-attention head-output nodes scored here.
Task-gradient baselines use the IOI margin.

\paragraph{Conditional and structural controls.}
Conditional energy ranks by $\energy(\dz_{u\mid\Sset})$ using the same primary set, ablation machinery, and Fisher geometry as \coax, but omits the intact effect.
Conditional gradient recomputes attribution after primary-set ablation, testing whether the intervened state alone exposes backups.
We also compare co-activation and the weight-space projection kernel~\citep{yamagiwa2026projection}.
A task-specific historical control ranks heads by the change in direct IO$-$S logit attribution after primary removal; because this mirrors the readout used to characterize the IOI backup branch, we report it separately in Appendix~\ref{app:historical-dla}.
These controls isolate conditioning, the intact-to-removed contrast, and the scoring representation.

\clearpage
\section{Backup Recovery and Mechanistic Validation}
\label{app:discovery}

This appendix expands the evidence behind Section~\ref{sec:discovery} in the same order as the main text.
We first report full backup-recovery results, then test whether recovery is specific to the removed primary set and whether the recovered branch causally carries repair.
We close with diagnostics of conditional growth and robustness to reasonable protocol variation.

\subsection{Backup Recovery and Statistical Validation}
\label{app:sig-group}

Figure~\ref{fig:separation} gives the distributional view of the main recovery result.
The documented backups lie near the bottom of the intact-state single-ablation ranking but move sharply upward under conditional growth, producing the inversion summarized by Table~\ref{tab:backup}.

\begin{figure}[ht]
\centering
\includegraphics[width=0.86\textwidth]{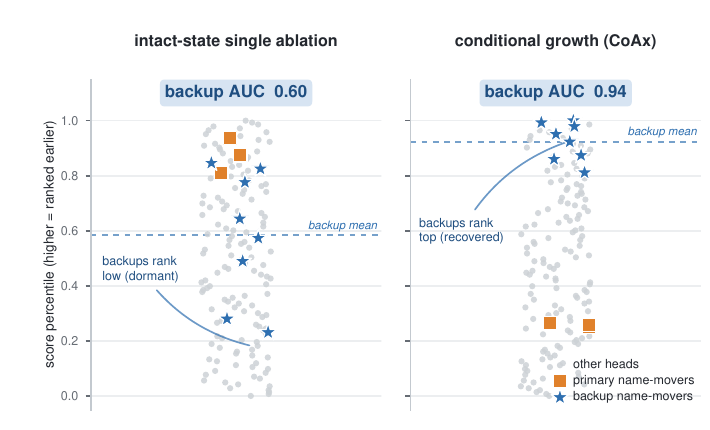}
\caption{\textbf{The same backups under intact-state and conditional scoring.}
Each point is one of the $144$ GPT-2-small heads; the $141$ non-primary heads are ranked for evaluation and the three primary name-movers are shown only for context.
The eight documented backups (blue stars) rank near the bottom under intact-state single-ablation saliency (AUC $0.60$) but near the top under \coax conditional growth (AUC $0.94$).
Primary heads are excluded from the backup ROC-AUC.}
\label{fig:separation}
\end{figure}

\subsubsection{Recovery across prompt seeds}

The headline result is stable across the four declared prompt seeds rather than being driven by one prompt sample.
\coax leads every baseline on all four seeds and reaches $0.941{\pm}0.004$ backup ROC-AUC (Table~\ref{tab:backup-seeds}).
AtP$^\star$ GradDrop remains the strongest intact-state attribution baseline and conditional gradient the strongest removed-only control on every seed.

\begin{table}[ht]
\centering
\small
\setlength{\tabcolsep}{4pt}
\caption{\textbf{Backup ROC-AUC across prompt seeds} (GPT-2-small IOI; $96$ prompts per seed).
\coax leads on all four seeds and has the smallest spread.}
\label{tab:backup-seeds}
\begin{tabular}{lccccc}
\toprule
\rowcolor{hdr}
score & s1 & s15 & s22 & s8 & mean $\pm$ std \\
\midrule
single ablation           & 0.599 & 0.614 & 0.604 & 0.597 & $0.603 \pm 0.007$ \\
AtP                       & 0.633 & 0.656 & 0.656 & 0.655 & $0.650 \pm 0.010$ \\
conditional gradient      & 0.759 & 0.770 & 0.768 & 0.772 & \secnd{$0.767 \pm 0.005$} \\
activation-space EAP-IG   & 0.623 & 0.573 & 0.582 & 0.547 & $0.581 \pm 0.027$ \\
AtP$^\star$ GradDrop      & 0.652 & 0.677 & 0.666 & 0.665 & $0.665 \pm 0.009$ \\
conditional energy        & 0.757 & 0.764 & 0.758 & 0.755 & $0.758 \pm 0.004$ \\
\rowcolor{ourrow}
\coax (ours)              & 0.945 & 0.944 & 0.934 & 0.943 & \bestauc{0.941}{0.004} \\
\bottomrule
\end{tabular}
\end{table}

ROC-AUC measures the full ranking; top-$k$ recovery gives the corresponding small-set view.
The \coax ranking contains $4/8$ documented backups in its top $8$, $4/8$ in the top $10$, $5/8$ in the top $15$, and $6/8$ in the top $20$.
Because only $8/141$ candidates are positive, average precision gives a stricter top-of-ranking check: \coax reaches $0.557{\pm}0.011$ versus $0.098{\pm}0.002$ for AtP$^\star$ GradDrop, with precision@$8$ of $0.50$ versus $0.00$.
Across prompt seeds, the full \coax rankings have pairwise Spearman correlation $0.913{\pm}0.032$, with Jaccard overlap $0.852{\pm}0.105$ at top-$8$ and $0.806{\pm}0.056$ at top-$20$.
The corresponding hypergeometric probabilities under random ranking are $2.8{\times}10^{-4}$, $8.1{\times}10^{-4}$, $3.2{\times}10^{-4}$, and $9.2{\times}10^{-5}$.

\subsubsection{Matched controls and statistical tests}
\label{app:sig}

Against the strongest intact-state attribution baseline, AtP$^\star$ GradDrop, \coax improves backup AUC from $0.665$ to $0.941$.
A DeLong test on seed-averaged head scores gives $p=7.8{\times}10^{-5}$; bootstrap $95\%$ CIs are $[0.885,0.986]$ for \coax and $[0.471,0.828]$ for AtP$^\star$.
Treating each prompt seed as one paired observation gives a mean AUC gap of $+0.276{\pm}0.012$, $t_3=46.90$, $p=2.1{\times}10^{-5}$.
These tests quantify variation over prompts within this fixed model and benchmark rather than variation over models or circuits.

The more diagnostic comparisons hold the supplied primary set fixed.
Conditional energy uses the same prompts, primary set, finite ablations, and Fisher geometry as \coax but omits the intact-state term; it reaches $0.758{\pm}0.004$, with DeLong $p=3.8{\times}10^{-4}$ against \coax.
Conditional gradient also scores the primary-ablated model but replaces the finite full-distribution effect with task-gradient attribution; it reaches $0.767{\pm}0.005$, with paired DeLong $p=0.015$ against \coax.
Thus reading the intervened state helps, but the intact-to-removed contrast supplies the larger gain.

Two exact non-parametric checks lead to the same conclusion.
A $10{,}000$-shuffle label-permutation test places the observed AUC entirely outside the null distribution (null mean $0.50$, $95$th percentile $0.67$, $p<10^{-4}$), and the top-$k$ hypergeometric tests above reject chance at every reported cutoff.

\subsubsection{Two-state gradient contrasts recover the conditional signal}
\label{app:condatp}

A natural alternative is to re-run an existing attribution method after removing the supplied primary set, or to take the same intact-to-removed difference of that attribution score.
Table~\ref{tab:condatp} tests both possibilities on the same $96$ IOI prompts, four prompt seeds, primary set, and candidate pool used for backup recovery.
Moving AtP or AtP$^\star$ GradDrop to the primary-ablated model improves AUC from $0.650$ to $0.767$ and from $0.665$ to $0.779$, respectively.
Taking the corresponding intact-to-removed differences raises recovery further, to $0.911$ for AtP and $0.938$ for AtP$^\star$, close to finite-displacement \coax at $0.941$.
This supports the central mechanistic claim: the informative signal is the change in causal responsibility across states, rather than importance in the removed state alone.
The estimands remain distinct.
The gradient contrasts are task-direction, local-linear approximations that require backward passes, whereas \coax measures the exact finite zero-ablation displacement of the predictive distribution with forward passes and admits the interaction decomposition in Proposition~\ref{prop:exact}.

\begin{table}[ht]
\centering
\small
\setlength{\tabcolsep}{6pt}
\caption{\textbf{Two-state gradient contrasts recover much of the same conditional signal.}
Documented-backup ROC-AUC on GPT-2-small IOI (mean$\pm$std over four prompt seeds).
``Conditional'' scores the primary-ablated model; ``two-state'' subtracts the corresponding intact-state score.}
\label{tab:condatp}
\begin{tabular}{lccc}
\toprule
\rowcolor{hdr}
score & intact & conditional & two-state \\
\midrule
AtP (grad$\times$act) & $0.650{\pm}0.010$ & $0.767{\pm}0.005$ & $0.911{\pm}0.002$ \\
AtP$^\star$ GradDrop & $0.665{\pm}0.009$ & $0.779{\pm}0.004$ & $0.938{\pm}0.004$ \\
\rowcolor{ourrow}
finite ablation displacement & $0.603{\pm}0.007$ & $0.758{\pm}0.004$ & \bestauc{0.941}{0.004} \\
\bottomrule
\end{tabular}
\end{table}

\subsubsection{Co-activation as an input-side control}
\label{app:coact}

Co-activation gives a complementary intact-state view.
Ranking heads by the mean absolute correlation between their attention-output norms and those of the primary name-movers reaches $0.93$ backup AUC, close to \coax's $0.941$.
This shows that the documented backups belong to a coherent co-firing neighborhood even before intervention.

The two scores nevertheless identify different properties.
Co-activation records intact-state co-firing, whereas \coax asks whether a candidate's causal effect grows specifically after the supplied primary set is removed.
Freezing co-activation's selected set causes a larger raw IOI-margin loss than freezing the \coax set ($1.97{\pm}0.06$ versus $0.90{\pm}0.07$; Appendix~\ref{app:freezesel}), but it gives substantially worse backup-role knockout matching (Appendix~\ref{app:oracleko}).
High label overlap or large raw freezing impact therefore does not by itself establish that a selector has isolated the missing causal role.

\subsubsection{Historical direct-logit readout}
\label{app:historical-dla}

The documented IOI backup branch was originally characterized using direct effects after name-mover knockout~\citep{wang2022ioi}, so we separately test the corresponding task-specific ranking: the change in direct attribution to the IO$-$S logit direction after removing the primary set.
As expected for a readout closely aligned with how the reference labels were characterized, it recovers those labels slightly better than \coax ($0.971{\pm}0.003$ versus $0.941{\pm}0.004$ ROC-AUC; average precision $0.620{\pm}0.023$ versus $0.557{\pm}0.011$). Its selected set does not reproduce the same causal role: the top-8 overshoots the freezing readout ($1.62$ versus $0.90$ for \coax), is farther from the documented knockout in IOI-margin distance ($1.33$ versus $0.99$; paired difference $+0.344$, $95\%$ CI $[+0.235,+0.450]$), and repeatedly includes the documented negative name-mover $[10,7]$. Thus reproducing the historical label readout is not identical to recovering the missing branch's causal role.

\subsubsection{Candidate-depth control}
\label{app:layer-control}

All eight documented backups lie in layers $9$--$11$, making depth a strong structural prior.
We therefore repeat the comparison under the \emph{headline} last-position/full-vocabulary protocol while restricting the candidate pool to that late-layer band.
\coax remains at $0.909{\pm}0.002$ AUC within layers $9$--$11$, compared with $0.691{\pm}0.011$ for intact single-ablation saliency; its score--layer correlation is also only $+0.114{\pm}0.011$, smaller than for the other high-scoring structural and attribution selectors (Table~\ref{tab:layer-control}).
Thus the recovery is not explained by a preference for late-layer heads.

\begin{table}[ht]
\centering
\small
\setlength{\tabcolsep}{7pt}
\caption{\textbf{Candidate-depth control under the headline protocol.}
The first column restricts candidates to layers $9$--$11$; the second reports Spearman correlation between score and layer over the full candidate pool.}
\label{tab:layer-control}
\begin{tabular}{lcc}
\toprule
\rowcolor{hdr}
selector & L9--L11 AUC $\uparrow$ & $\rho$(score, layer) \\
\midrule
\rowcolor{ourrow}
\coax & $0.909{\pm}0.002$ & $+0.114{\pm}0.011$ \\
conditional direct-logit change & $0.887{\pm}0.016$ & $+0.175{\pm}0.017$ \\
co-activation & $0.776{\pm}0.010$ & $+0.436{\pm}0.003$ \\
IO-attention structural read & $0.788{\pm}0.004$ & $+0.483{\pm}0.002$ \\
AtP$^\star$ GradDrop & $0.739{\pm}0.004$ & $-0.309{\pm}0.011$ \\
single ablation & $0.691{\pm}0.011$ & $-0.328{\pm}0.008$ \\
\bottomrule
\end{tabular}
\end{table}

\subsubsection{A primary-removal-free reference construction}
\label{app:noncirc}

The eight documented backups are a useful external reference, but their original characterization involved name-mover knockout~\citep{wang2022ioi}.
To verify that the intact-state blind spot is not an artifact of that labeling route, we construct a second GPT-2-small reference set without observing any primary-removal response.
We first identify a name-mover family using two intact structural criteria: final-position attention to the IO name and an OV copy score that tests whether the head writes a name representation back toward the same token.
Within this family, the three largest intact direct-logit contributors define the primaries; candidates within a $0.75$ guard band of the weakest primary are excluded, and the remaining five weak contributors form the dormant reference set.
This construction recovers exactly the three documented primary name-movers; three of its five dormant members are documented backups, while the remaining two are $[11,1]$ and $[11,6]$.
Varying the family size from $12$ to $28$ and the guard band from $0.70$ to $0.90$ yields $20$ constructions that all pass the independent causal check; across them, \coax remains at $0.970$--$0.992$ ROC-AUC while intact single ablation remains at $0.394$--$0.483$.

\begin{table}[ht]
\centering
\small
\setlength{\tabcolsep}{6pt}
\caption{\textbf{Recovery of a primary-removal-free dormant reference set} (GPT-2-small; five positives, four prompt seeds).
The reference labels use only intact structural reads and intact direct contribution; no quantity measured after primary removal enters their construction.}
\label{tab:noncirc}
\begin{tabular}{lcc}
\toprule
\rowcolor{hdr}
score & ROC-AUC $\uparrow$ & average precision $\uparrow$ \\
\midrule
single ablation & $0.444{\pm}0.004$ & $0.036{\pm}0.000$ \\
co-activation & $0.955{\pm}0.001$ & $0.312{\pm}0.004$ \\
conditional direct-logit change & $0.985{\pm}0.003$ & $0.599{\pm}0.050$ \\
\rowcolor{ourrow}
\coax & $0.983{\pm}0.001$ & $0.586{\pm}0.006$ \\
\bottomrule
\end{tabular}
\end{table}

The resulting labels also pass an intervention-based check that is not used in their construction.
Against controls matched exactly in layer and closely in intact ablation energy, the dormant set has larger output-norm wake-up ($1.159$ versus $1.091$), direct-logit hand-off ($+0.187$ versus $+0.047$), conditional causal drop ($+0.197$ versus $-0.062$), and freeze-carried repair ($+0.514$ versus $+0.134$).
Thus the conditional recovery signal persists when the positive set is defined independently of primary removal, while the two independently surfaced non-reference heads $[11,1]$ and $[11,6]$ receive convergent structural and intervention evidence.

\subsection{Recovery Is Specific to the Removed Primary Set}
\label{app:specificity}
\label{app:fishermatch}

Section~\ref{sec:specificity} asks whether the recovered ranking follows the identity of the removed circuit rather than generic intervention severity.
We test this with alternative documented and matched primary sets, then remove LayerNorm rescaling as a potential source of background growth.

\subsubsection{Alternative and geometry-matched primary sets}

Table~\ref{tab:specificity} varies the supplied primary set while keeping the backup labels and evaluation protocol fixed.
Matching behavioral damage, depth, or task relevance does not reproduce the ranking induced by the documented name-movers.
The correct set gives $0.941{\pm}0.004$ backup AUC, whereas a damage-matched random set gives $0.426{\pm}0.129$, a same-layer control $0.536{\pm}0.196$, and documented S-inhibition and induction sets $0.688{\pm}0.010$ and $0.359{\pm}0.017$.
Adding off-circuit heads to the correct set degrades recovery gradually from $0.92$ with one contaminating head to $0.74$ with four.

\begin{table}[ht]
\centering
\small
\setlength{\tabcolsep}{5pt}
\caption{\textbf{Backup recovery follows primary-set identity.}
GPT-2-small IOI backup AUC under alternative supplied primary sets (mean$\pm$std over four prompt seeds; contamination entries are means).
Matching behavioral damage or depth does not reproduce the documented backup ranking, while contamination of the correct set degrades recovery gradually.}
\label{tab:specificity}
\begin{tabular}{lcc}
\toprule
\rowcolor{hdr}
primary set $\Sset$ & joint damage & backup AUC \\
\midrule
\rowcolor{ourrow}
documented name-movers & $+0.284$ & \bestauc{0.941}{0.004} \\
\midrule
damage-matched random & $+0.295$ & $0.426{\pm}0.129$ \\
random & $-0.152$ & $0.492{\pm}0.113$ \\
same-layer, different heads & $+0.085$ & $0.536{\pm}0.196$ \\
documented S-inhibition set & $+1.309$ & $0.688{\pm}0.010$ \\
documented induction set & $+1.336$ & $0.359{\pm}0.017$ \\
\midrule
true set $+1/+2/+3/+4$ off-circuit heads
& --- & $0.92/0.84/0.81/0.74$ \\
\bottomrule
\end{tabular}
\end{table}

We make the negative control harder by matching wrong three-head sets in the score's own output geometry.
From $5{,}000$ random triples scored on calibration prompts, we retain the $16$ closest to the documented primary set in behavioral damage, answer-position Fisher displacement $\energy(\dz_\Sset)$, and mean layer index, then evaluate them on held-out prompts.
The matched wrong sets remain below chance on average: $0.397{\pm}0.133$ and $0.405{\pm}0.129$ AUC on two held-out prompt sets, versus $0.944$ and $0.934$ for the documented primaries.
Even the closest behavioral-damage matches score only $0.319$ and $0.413$ on the two evaluations. Matching both behavioral damage and the output displacement used by the score therefore does not reproduce the backup ranking.

\begin{table}[ht]
\centering
\small
\setlength{\tabcolsep}{8pt}
\caption{\textbf{Geometry-matched primary-set control.}
Backup AUC on two held-out prompt sets after selecting $16$ wrong triples from a $5{,}000$-set search, matched to the documented primary set in behavioral damage, Fisher displacement, and depth.
For the matched controls, mean$\pm$std is across the selected triples.}
\label{tab:fishermatch}
\begin{tabular}{lcc}
\toprule
\rowcolor{hdr}
primary set & held-out A & held-out B \\
\midrule
\rowcolor{ourrow}
documented name-movers & $0.944$ & $0.934$ \\
geometry-matched wrong sets
& $0.397{\pm}0.133$ & $0.405{\pm}0.129$ \\
\bottomrule
\end{tabular}
\end{table}

\subsubsection{LayerNorm freezing}
\label{app:lnfreeze}

Primary-set ablation also changes downstream normalization statistics and can rescale the remaining computation.
We isolate this channel by recording the intact-state centering mean and normalization denominator for every GPT-2-small LayerNorm and replaying ablated forwards with these statistics partially or fully fixed.
A \emph{scale freeze} fixes the denominator, while a \emph{full freeze} also fixes the centering mean.
Replaying the recorded statistics on an otherwise intact forward reproduces the original logits to within $10^{-3}$ before either control is used.

Backup AUC rises from $0.941{\pm}0.004$ in the unmodified model to $0.983{\pm}0.002$ under scale freeze and $0.978{\pm}0.002$ under full freeze.
Under scale freeze, the documented backups retain $0.81{\pm}0.01$ of their conditional growth ($0.076{\pm}0.003\to0.062{\pm}0.003$), while mean non-backup growth shifts from approximately zero to negative ($-0.001{\pm}0.001\to-0.030{\pm}0.002$).
Normalization rescaling therefore does not create the recovered backup signal; removing it suppresses background growth and sharpens the separation.

\subsection{Causal Validation of the Recovered Branch}
\label{app:real-group}

Backup-label recovery establishes ranking agreement, but the mechanistic claim requires the recovered branch to respond to primary-set failure and causally carry the resulting repair.
We first situate that branch in the documented IOI circuit, then give the full measurements behind Figure~\ref{fig:mechanism} and an independent structural check.

\subsubsection{Circuit context: the documented fallback branch}
\label{app:circuitviz}

Figure~\ref{fig:routes} places the primary and backup name-movers inside the documented IOI information flow.
The earlier duplicate-token, induction, and S-inhibition heads route the relevant name information toward the final position, after which \primary{primary} and \backup{backup} name-movers form parallel late write paths to the answer.
The diagram provides circuit context; the intervention measurements below establish when the backup branch becomes causally active.

\begin{figure}[ht]
\centering
\includegraphics[width=\textwidth]{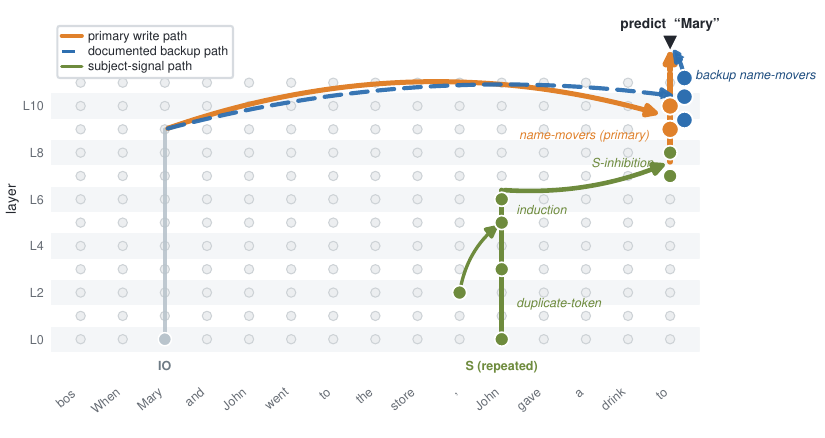}
\caption{\textbf{The documented IOI fallback branch in token$\times$layer context.}
Each column is a prompt token and each row a layer; gray nodes denote the residual stream.
Duplicate-token, induction, and S-inhibition heads carry the earlier subject-related signal (olive), while the primary name-movers (orange) and documented backup name-movers (blue, dashed) form parallel late write paths from the IO name to the answer position.
The routes follow the documented IOI circuit~\citep{wang2022ioi}; they are circuit structure rather than per-edge \coax attributions.
Figure~\ref{fig:mechanism} establishes that the backup branch is recruited after primary-set removal.}
\label{fig:routes}
\end{figure}

\subsubsection{Hand-off and graded wake-up}
\label{app:mechanism}

\paragraph{Direct-logit hand-off.}
For Figure~\ref{fig:mechanism}a, direct-logit attribution projects each head's residual-stream write onto $W_U[\mathrm{IO}]-W_U[\mathrm{S}]$ at the answer position using the final-LayerNorm statistics from that forward.
We report branch totals because the figure's edge width represents the contribution of the whole branch.
In the intact model, the three \primary{primary} name-movers contribute $+2.519{\pm}0.169$, the eight documented \backup{backups} $+0.567{\pm}0.048$, and the remaining $133$ heads $-0.773{\pm}0.151$ to the IO$-$S logit.
After \primary{primary}-set removal, the \primary{primary} contribution is zero by construction, the documented \backup{backup} branch rises to $+1.680{\pm}0.149$, and the remaining heads shift to $+0.392{\pm}0.141$.
The IOI margin nevertheless changes only from $2.574{\pm}0.083$ to $2.290{\pm}0.110$.
The redistribution inside the circuit is therefore much larger than the behavioral loss.

\paragraph{Graded wake-up.}
Figure~\ref{fig:mechanism}b removes the three primary name-movers progressively in descending order of their single-head IOI effect.
For each $k\in\{0,1,2,3\}$, we measure the output-norm ratio and the additional IOI-margin drop caused by ablating a candidate after the first $k$ primaries have already been removed.
The documented backups grow monotonically on both readouts, whereas matched random heads remain approximately flat.

For the \coax curve, each prompt seed selects its top-$8$ once using the full three-head primary set and holds that set fixed while $k$ varies.
The response therefore cannot be explained by reselecting heads at each intervention level.
The recovered set follows the same monotone direction as the documented branch.

\begin{table}[ht]
\centering
\small
\setlength{\tabcolsep}{3.2pt}
\caption{\textbf{Graded wake-up behind Figure~\ref{fig:mechanism}b.}
Mean$\pm$std over four prompt seeds.
The \coax top-$8$ is selected once with the full primary set removed and held fixed across $k$; one matched-random set per prompt seed is likewise held fixed across the sweep.}
\label{tab:mechanism}
\begin{tabular}{c|ccc|ccc}
\toprule
\rowcolor{hdr}
& \multicolumn{3}{c|}{output-norm ratio}
& \multicolumn{3}{c}{conditional causal effect} \\
\rowcolor{hdr}
$k$
& documented & \coax & random
& documented & \coax & random \\
\midrule
0
& $1.000{\pm}.000$
& $1.000{\pm}.000$
& $1.000{\pm}.000$
& $0.046{\pm}.001$
& $0.112{\pm}.020$
& $0.017{\pm}.022$ \\
1
& $1.031{\pm}.015$
& $1.024{\pm}.012$
& $1.001{\pm}.002$
& $0.059{\pm}.013$
& $0.114{\pm}.029$
& $0.010{\pm}.020$ \\
2
& $1.082{\pm}.001$
& $1.062{\pm}.005$
& $1.004{\pm}.002$
& $0.083{\pm}.003$
& $0.118{\pm}.021$
& $0.010{\pm}.026$ \\
3
& $1.150{\pm}.002$
& $1.118{\pm}.008$
& $1.006{\pm}.004$
& $0.102{\pm}.006$
& $0.136{\pm}.020$
& $0.009{\pm}.036$ \\
\bottomrule
\end{tabular}
\end{table}

\subsubsection{Dose response to partial primary removal}
\label{app:dose}

The main experiment uses complete zero ablation of the primary set.
As a continuous-strength check, we scale every primary head output by $1-\alpha$ for $\alpha\in\{0,0.25,0.5,0.75,1\}$.
The documented backups increase monotonically on both wake-up readouts as the intervention strengthens, while matched random heads remain essentially flat or move in the opposite direction (Table~\ref{tab:dose}).
This shows that recruitment is graded in intervention strength rather than appearing only at the endpoint ablation.

\begin{table}[ht]
\centering
\small
\setlength{\tabcolsep}{4pt}
\caption{\textbf{Dose response to partial primary removal} (GPT-2-small IOI; four prompt seeds).
Wake-up is the output-norm ratio relative to the intact state; conditional drop is the additional IOI-margin loss from ablating a candidate after the partial primary intervention.}
\label{tab:dose}
\begin{tabular}{cccccc}
\toprule
\rowcolor{hdr}
$\alpha$ & IOI margin & backup wake & random wake & backup cond. drop & random cond. drop \\
\midrule
$0.00$ & $2.583$ & $1.000$ & $1.000$ & $+0.047$ & $-0.004$ \\
$0.25$ & $2.502$ & $1.038$ & $1.002$ & $+0.062$ & $-0.009$ \\
$0.50$ & $2.441$ & $1.077$ & $1.003$ & $+0.076$ & $-0.014$ \\
$0.75$ & $2.398$ & $1.114$ & $1.005$ & $+0.090$ & $-0.019$ \\
$1.00$ & $2.357$ & $1.150$ & $1.006$ & $+0.104$ & $-0.023$ \\
\bottomrule
\end{tabular}
\end{table}

\subsubsection{Selector-specific causal freezing}
\label{app:freezesel}

The graded curves establish intervention responsiveness; freezing tests whether that response is causally load-bearing.
With the primary set already ablated, we freeze each selector's top-$8$ heads at their intact-state final-position outputs so that those heads cannot respond to the intervention.
For IOI margin $m$, we report the freeze-induced loss
\begin{equation}
L_{\mathrm{freeze}}(X)
=
m(\text{ablate }\Sset)
-
m(\text{ablate }\Sset,\text{ freeze }X).
\label{eq:freezeloss}
\end{equation}
We report this absolute quantity so every selector is compared under the same freezing intervention; it tests whether a selected set is load-bearing, not whether a larger value is a better circuit match.

Freezing the \coax-selected set causes a $0.90{\pm}0.07$ IOI-margin loss, compared with $0.03{\pm}0.08$ for conditional energy and $0.27{\pm}0.37$ for a size-matched random set.
The documented backups themselves give $0.57{\pm}0.09$ under the same freezing intervention.
Co-activation and AtP$^\star$ produce larger raw losses, $1.97{\pm}0.06$ and $1.36{\pm}0.03$, respectively; this shows why freezing magnitude is a causal load-bearing diagnostic rather than a partition of total repair or a standalone circuit-matching criterion.
Circuit completion and knockout in Appendix~\ref{app:attribution} provide the complementary specificity test.

\begin{table}[ht]
\centering
\small
\caption{\textbf{Freeze-induced IOI-margin loss for each selector's top-$8$} (GPT-2-small; mean$\pm$std over four prompt seeds).
All rows use the same intervention: the selected heads are fixed at their intact-state outputs after the primary set has been ablated.
``Overlap'' is the mean number of documented backups in each selected set.}
\label{tab:freezesel}
\begin{tabular}{lcc}
\toprule
\rowcolor{hdr}
frozen set & IOI-margin loss & overlap \\
\midrule
documented backups \emph{(oracle reference)} & $0.57{\pm}0.09$ & $8.0/8$ \\
\rowcolor{ourrow}
\coax growth & $0.90{\pm}0.07$ & $4.0/8$ \\
amplification-ratio variant & $0.69{\pm}0.13$ & $5.75/8$ \\
conditional energy & $0.03{\pm}0.08$ & $1.0/8$ \\
co-activation & $1.97{\pm}0.06$ & $3.0/8$ \\
AtP$^\star$ GradDrop & $1.36{\pm}0.03$ & $0.0/8$ \\
random & $0.27{\pm}0.37$ & $0.25/8$ \\
\bottomrule
\end{tabular}
\end{table}

The overlap column is descriptive rather than an exhaustive precision measure: the documented eight define the known backup branch but do not imply that no other heads participate in compensation.
The freezing readout therefore asks whether a selected set is causally load-bearing under primary failure; the completion and knockout tests ask the stricter question of whether it matches the missing branch's causal role.

\subsubsection{Independent structural corroboration}
\label{app:sigfilter}

The causal readouts above are intervention-based, so we additionally use a structural signal that \coax never observes.
A name-mover's defining attention pattern is to attend from the answer position back to the indirect-object name~\citep{wang2022ioi}.
Reading this final-position QK attention without ablation, the documented backups attend to the IO token $5.8\times$ more strongly than other heads on average ($0.19$ versus $0.03$).
This structural score alone gives $0.96$ backup ROC-AUC while correlating only weakly with \coax ($\rho=0.09$ Spearman).

Figure~\ref{fig:attn} makes the structural check explicit.
The primary and backup name-movers preferentially read the IO name rather than the repeated subject or filler tokens, while unrelated controls do not show the same pattern.
Because this read uses the known name-mover token relation, we use it as independent post-hoc corroboration rather than as a general discovery score.
Agreement between this ablation-free structural signature and the intervention-based evidence strengthens the mechanistic identification of the recovered branch.

\begin{figure}[ht]
\centering
\includegraphics[width=0.72\textwidth]{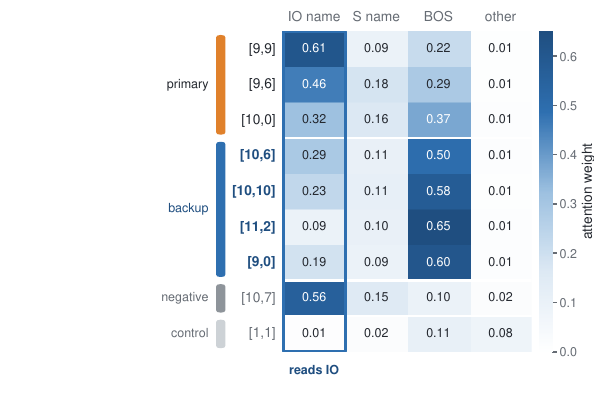}
\caption{\textbf{Independent structural read of the name-mover branch.}
Final-position attention to token roles, averaged over $96$ IOI prompts.
Documented primary and backup name-movers attend preferentially to the IO name, while control heads do not show the same pattern.
This ablation-free structural read gives $0.96$ backup ROC-AUC and correlates only weakly with the \coax score ($\rho=0.09$), providing complementary evidence for the recovered role.}
\label{fig:attn}
\end{figure}

\subsubsection{A concrete dormant backup}
\label{app:casestudy}

The aggregate pattern above can be seen directly in the documented backup name-mover $[10,6]$.
In the intact model, its direct IO$-$S contribution is small ($+0.07$), yet from the answer position it attends strongly to the IO name ($0.29$), matching the structural role of a name-mover.
After the primary name-movers are removed, its direct IO$-$S contribution rises to $+0.21$, while a random comparison head remains near zero ($+0.01$).
This head-level view connects the structural role to intervention-induced recruitment; the set-level freezing experiment in Appendix~\ref{app:freezesel} provides the corresponding causal validation of repair.

\begin{figure}[ht]
\centering
\includegraphics[width=0.96\textwidth]{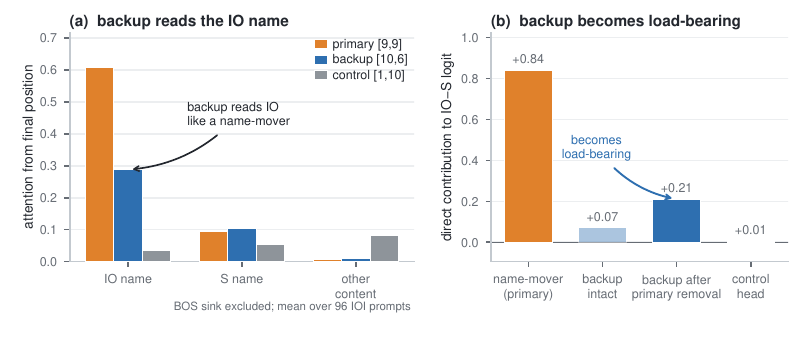}
\caption{\textbf{A concrete dormant backup in GPT-2-small IOI.}
\textbf{(a) Backup reads the IO name.} The documented backup $[10,6]$ attends strongly to the IO name from the answer position, matching the structural signature of a name-mover.
\textbf{(b) Its contribution grows after primary-set failure.} The backup's direct IO$-$S contribution increases from $+0.07$ in the intact state to $+0.21$ after the primary name-movers are removed, while a random comparison head remains near zero.}
\label{fig:casestudy}
\end{figure}

\subsection{What Drives Conditional Growth}
\label{app:keys-group}

\subsubsection{The intact-to-removed contrast and output geometry}
\label{app:components}

Table~\ref{tab:design} separates two design choices: which states are compared and how output displacement is measured.
The dominant gain comes from the state comparison.
Intact-state single-ablation energy reaches $0.603$ backup AUC; conditional energy in the primary-ablated model raises this to $0.758$; contrasting the two states reaches $0.941$.
Conditional energy alone retains components that are large in either state, whereas subtraction removes this pre-existing importance and isolates intervention-induced growth.
A normalized alternative, $\energy(\dz_{u\mid\Sset})/\energy(\dz_u)$, reaches $0.984$ AUC and retrieves all eight documented backups by rank $20$.
Its seed-averaged ranking is not statistically distinguishable from signed growth under paired DeLong ($p=0.064$).
We retain signed growth as the default because Proposition~\ref{prop:exact} applies directly to its underlying effect difference, it avoids unstable division by near-zero intact effects, and it transfers unchanged to the mechanism panel.

Holding the two-state comparison fixed, plain $\ell_2$ displacement already reaches $0.905$ AUC.
Centering alone gives $0.903$, probability weighting alone $0.909$, and the full centered Fisher geometry $0.941$.
The remaining gain therefore belongs to the combined output geometry rather than either operation in isolation.

\begin{table}[ht]
\centering
\small
\caption{\textbf{Component analysis of \coax} (GPT-2-small; means over four prompt seeds).
The top block varies the state comparison; the bottom block varies the output geometry while retaining signed growth.}
\label{tab:design}
\begin{tabular}{lccc}
\toprule
\rowcolor{hdr}
score & backup AUC & recall@$8$ & recall@$20$ \\
\midrule
\rowcolor{hdr}
\multicolumn{4}{l}{\emph{state comparison}} \\
intact-state energy, $\energy(\dz_u)$ & $0.603$ & $0/8$ & $0/8$ \\
conditional energy, $\energy(\dz_{u\mid\Sset})$ & $0.758$ & $1/8$ & $3.0/8$ \\
amplification ratio (auxiliary variant), $\energy(\dz_{u\mid\Sset})/\energy(\dz_u)$ & $0.984$ & $5.75/8$ & $8.0/8$ \\
\rowcolor{ourrow}
signed growth (default), $\energy(\dz_{u\mid\Sset})-\energy(\dz_u)$ & $0.941$ & $4.0/8$ & $6.2/8$ \\
\midrule
\rowcolor{hdr}
\multicolumn{4}{l}{\emph{output geometry} (signed growth throughout)} \\
plain $\ell_2$ & $0.905$ & & \\
centering only & $0.903$ & & \\
probability weighting only & $0.909$ & & \\
\rowcolor{ourrow}
centered Fisher (full) & \best{$0.941$} & & \\
\bottomrule
\end{tabular}
\end{table}

\subsubsection{Controlled test of dormancy and task alignment}
\label{app:synth}

The IOI benchmark provides documented mechanistic labels but only eight backup heads.
We therefore use a controlled construction for two narrower questions: whether an intact-state score can identify a near-dormant backup, and whether a task-direction gradient remains informative when repair is not aligned with that direction.
This experiment isolates score behavior rather than serving as an additional realism benchmark.

\paragraph{Construction.}
We work in an $r=192$ output space with unit answer direction $e$ and Gaussian observation noise $\sigma=0.05$.
Four \primary{primaries} have large intact-state effects and directions partially aligned with $e$.
Each of $100$ \backup{backups} shadows a \primary{primary} direction but is near dormant in the intact state ($a_b^{\mathrm{intact}}\approx0.04$) and becomes active after primary removal ($a_b^{\mathrm{cond}}\approx0.42$).
A further $100$ \irrelevant{inert units} have random directions, small intact-state effects, and no conditional change.
Primaries are excluded from the classification task.

We compare intact-state energy, an intact-state answer-gradient proxy, the corresponding conditional gradient, and \coax signed growth.
The two intact-state scores sit at or below chance, whereas conditioning makes the planted backups identifiable (Table~\ref{tab:synth}).
\coax reaches $0.90{\pm}0.02$ AUC and the conditional gradient $0.85{\pm}0.04$ in the default answer-aligned construction.
Sweeping observation noise from $\sigma=0.03$ to $0.16$ reduces \coax from $0.95$ to $0.63$, while the intact-state scores remain near chance.

\begin{table}[ht]
\centering
\small
\setlength{\tabcolsep}{7pt}
\caption{\textbf{Dormancy isolates the value of conditioning.}
Controlled benchmark with $100$ conditionally active backups and $100$ inert units over $40$ independent trials.
Intact-state scores remain at or below chance, while both conditional scores recover the planted backups.}
\label{tab:synth}
\begin{tabular}{lcc}
\toprule
\rowcolor{hdr}
score & ROC-AUC\,$\uparrow$ & DeLong vs.\ \coax \\
\midrule
\rowcolor{ourrow}
\coax growth & \bestauc{0.90}{.02} & --- \\
intact-state energy & $0.42{\pm}.05$ & $p<10^{-15}$ \\
intact-state gradient proxy & $0.51{\pm}.04$ & $p<10^{-15}$ \\
conditional gradient & $0.85{\pm}.04$ & n.s.\ (default alignment) \\
\bottomrule
\end{tabular}
\end{table}

\paragraph{Task alignment.}
Let $\beta$ denote the fraction of a backup's effect energy aligned with the answer direction.
Holding the total conditional effect fixed while varying $\beta$ leaves \coax at approximately $0.90$ AUC, whereas the answer-direction conditional gradient falls from $0.95$ at $\beta=0.9$ to $0.59$ at $\beta=0.1$ (Table~\ref{tab:synth-sweep}).
The two conditional scores therefore answer different questions: conditional growth reads the full output displacement, while the gradient reads its task-aligned projection.

The documented IOI backups also lie in this low-alignment regime.
Across the eight heads, the measured IO$-$S energy fraction is below $0.01$ for every backup, with mean $0.002$ and across-head standard deviation $0.003$.
We use this measurement only to place the empirical IOI setting on the synthetic sweep, not as a discriminator of backup identity; it is consistent with the empirical gap between conditional gradient and full-distribution conditional growth on IOI.

\begin{table}[ht]
\centering
\small
\setlength{\tabcolsep}{5pt}
\caption{\textbf{Full-distribution growth is insensitive to answer alignment.}
Synthetic sweep over backup answer-alignment $\beta$ ($100$ backups, $40$ independent trials).
At fixed effect norm, \coax remains stable while the answer-direction conditional gradient degrades as repair moves away from the task direction.}
\label{tab:synth-sweep}
\begin{tabular}{lcccc}
\toprule
\rowcolor{hdr}
backup alignment $\beta$ & \coax AUC & cond.\ grad.\ AUC & intact grad.\ AUC & DeLong $p$ \\
\midrule
$0.9$ & $0.90$ & $0.95$ & $0.58$ & cond.\ grad.\ higher ($0.03$) \\
$0.5$ & $0.90$ & $0.87$ & $0.52$ & n.s.\ ($0.21$) \\
$0.3$ & $0.90$ & $0.75$ & $0.49$ & $\mathbf{<10^{-3}}$ \\
\rowcolor{ourrow}
$0.1$ & \best{$0.90$} & $0.59$ & $0.48$ & $\mathbf{<10^{-3}}$ \\
\bottomrule
\end{tabular}
\end{table}

\subsubsection{Interaction structure and empirical non-additivity}
\label{app:superadd}

The conditional effect change has an exact all-order interaction interpretation, but \coax does not enumerate those terms.
As a diagnostic, Figure~\ref{fig:interp} visualizes only the pairwise slice of the IOI circuit.
The primary and backup name-movers form a strong interaction block, consistent with each candidate's effect depending on whether the primary branch remains available.

\begin{figure}[ht]
\centering
\includegraphics[width=0.78\textwidth]{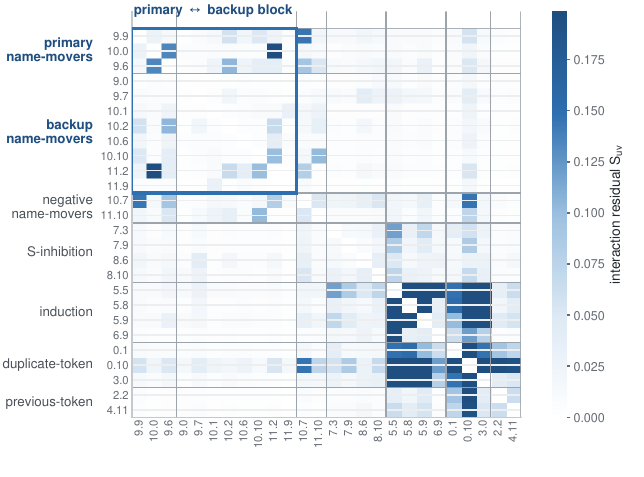}
\caption{\textbf{Pairwise slice of the IOI interaction structure.}
The magnitude of the pairwise co-ablation residual is shown for documented IOI circuit heads grouped by role.
Primary and backup name-movers form a strong block (blue outline).
This is only the order-two slice of the all-order aggregate in Proposition~\ref{prop:exact}; \coax does not compute this matrix.}
\label{fig:interp}
\end{figure}

The same dependence appears at the module level.
Jointly ablating the three primaries changes the IOI margin by only $0.28$ because the backup branch remains available.
Ablating the full primary-plus-backup module changes it by $1.30$, almost twice the $0.67$ predicted by summing the eleven single-head effects and $4.6\times$ the primary-only effect.
A size-matched random set is approximately additive.
The redundant module therefore cannot be reconstructed from isolated intact-state effects, matching the conditional dependence characterized by Proposition~\ref{prop:exact}.

\begin{table}[ht]
\centering
\small
\setlength{\tabcolsep}{8pt}
\caption{\textbf{Joint versus summed ablation} on GPT-2-small IOI (IOI-margin drop; intact margin $2.57$; means over four prompt seeds).
The primary-plus-backup module is strongly non-additive, whereas a size-matched random set is approximately additive.}
\label{tab:superadd}
\begin{tabular}{lccc}
\toprule
\rowcolor{hdr}
ablated set & $\sum$ single & joint & ratio \\
\midrule
$3$ primary name-movers & 0.30 & 0.28 & $0.9\times$ \\
$8$ documented backups & 0.37 & 0.30 & $0.8\times$ \\
\rowcolor{ourrow}
primaries $+$ backups & 0.67 & \best{1.30} & $\mathbf{1.9\times}$ \\
random, size-matched & \multicolumn{2}{c}{approximately additive} & $\approx1\times$ \\
\bottomrule
\end{tabular}
\end{table}

\subsection{Robustness}
\label{app:robust-group}

\subsubsection{Primary-set size}
\label{app:seedrobust}

\coax is conditioned on a supplied primary set, so we test whether strong recovery requires all three documented name-movers.
We order the primary heads by their own ablation energy and evaluate the first one, two, or all three.
Backup AUC is already $0.94$ with one primary, rises to $0.97$ with two, and is $0.94$ with the full set.
Together with the contamination experiment in Appendix~\ref{app:specificity}, this shows that recovery tolerates partial specification while remaining sensitive to off-circuit contamination.

\begin{table}[ht]
\centering
\small
\setlength{\tabcolsep}{5pt}
\caption{\textbf{Backup AUC versus supplied primary-set size} (mean over four prompt seeds).
A partial documented primary set already supports strong recovery.}
\label{tab:seedrobust}
\begin{tabular}{lccc}
\toprule
\rowcolor{hdr}
primary set & $|\Sset|{=}1$ & $|\Sset|{=}2$ & $|\Sset|{=}3$ \\
\midrule
\rowcolor{ourrow}
documented primaries & 0.94 & \best{0.97} & 0.94 \\
\bottomrule
\end{tabular}
\end{table}

\subsubsection{Prompt-template robustness}
\label{app:template}

The main IOI experiment uses one surface template.
Recomputing backup recovery on two alternative forms that exercise the same indirect-object name-moving behavior (``After \dots\ arrived at the \dots, \dots handed a \dots\ to'' and ``While \dots\ were at the \dots, \dots passed a \dots\ to'') gives AUC $0.93$ and $0.92$, versus $0.94$ for the standard template.
The recovery therefore persists across prompt forms rather than depending on one surface realization.

\subsubsection{Other additive units}
\label{app:ffn}

The formulation applies to any unit whose intervention induces a well-defined additive residual contribution.
As an implementation check, we instantiate \coax on contiguous slices of GPT-2-small MLP intermediate neurons ($96$ groups), zeroing each slice at the down-projection input and applying the same conditional score without modification.
The main causal benchmark remains head-level because the documented backup roles are defined at that granularity.

\clearpage
\section{Conditional Circuit Completion and Knockout}
\label{app:attribution}

This appendix expands Section~\ref{sec:apps} at the circuit level.
We first report the raw-headline \coax completion and matched controls, then give the full knockout comparison against the documented backup intervention.
Recovery, freezing, completion, and knockout all use the raw headline \coax ranking; no backup labels, task gradients, or task-direction filters enter the \coax selection.

\subsection{Circuit-Completion Controls}
\label{app:completeall}

The controlled primary circuit has incompleteness $0.474{\pm}0.030$, while the documented circuit reaches $0.300{\pm}0.027$.
Using the same raw headline ranking as discovery, the \coax top-$8$ reduces incompleteness to $0.236{\pm}0.037$; a size-matched random top-up gives $0.499{\pm}0.073$ (Table~\ref{tab:completeall}).
Conditional energy attains the smallest scalar gap ($0.181{\pm}0.027$), but its poor backup recovery and knockout match show that response magnitude alone need not identify the missing branch.
The raw \coax ranking therefore provides a direct completion result without introducing any task-direction selection step, while the knockout analysis below tests branch identity.

\begin{table}[ht]
\centering
\small
\setlength{\tabcolsep}{8pt}
\caption{\textbf{Circuit completion with the raw headline ranking} (GPT-2-small; mean$\pm$std over four prompt seeds; lower is better).
The \coax row uses the same last-position/full-vocabulary raw ranking as backup discovery, with no task-direction filter.}
\label{tab:completeall}
\begin{tabular}{lc}
\toprule
\rowcolor{hdr}
circuit & incompleteness $\downarrow$ \\
\midrule
primary circuit \emph{(no completion)} & $0.474{\pm}0.030$ \\
$+$random & $0.499{\pm}0.073$ \\
$+$conditional energy & \best{$0.181{\pm}0.027$} \\
\rowcolor{ourrow}
$+$\coax & $0.236{\pm}0.037$ \\
$+$amplification-ratio variant & $0.258{\pm}0.035$ \\
$+$AtP$^\star$ GradDrop & $0.309{\pm}0.031$ \\
$+$\coax next-$8$ & $0.342{\pm}0.040$ \\
$+$co-activation & $0.403{\pm}0.038$ \\
documented circuit \emph{(reference)} & $0.300{\pm}0.027$ \\
\bottomrule
\end{tabular}
\end{table}

A prompt-level paired bootstrap confirms that the gain is stable: relative to the primary-only circuit and a size-matched random completion, raw \coax lowers incompleteness with $95\%$ CIs $[-0.266,-0.211]$ and $[-0.300,-0.226]$, respectively.
The result uses exactly the same raw candidate ranking as discovery, so the circuit-level completion does not require task-direction information beyond the supplied primary set.

\subsection{Knockout Against the Documented Backup Intervention}
\label{app:oracleko}

We next reverse the completion test.
Each selector contributes eight heads, which are removed together with the primary name-movers and compared with removal of the documented primaries plus documented backups.
The random row averages $20$ independently sampled size-matched sets per prompt seed.
We measure the mean absolute difference in IOI margin, KL divergence from the documented knockout at the answer position, and KL from the intact model on unrelated text.
Task accuracy alone would give the wrong ordering: the documented knockout has IOI accuracy $0.714$ and \coax $0.779$, whereas the \coax next-$8$ and AtP$^\star$ selections drive accuracy much lower, to $0.552$ and $0.180$, despite lying farthest from the documented intervention.
This inversion is why we judge knockout by reproducing the documented causal intervention rather than by maximizing task damage.

Signed \coax and the amplification-ratio variant are tied in IOI-margin distance in Table~\ref{tab:oracleko} ($0.99$ each); their paired difference is $0.003$ with $95\%$ CI $[-0.087,0.095]$.
The ratio variant is closer in output KL, whereas signed growth retains the exact interaction interpretation and is the score evaluated across the broader mechanism panel.
Conditional energy, a continuation of the \coax ranking, AtP$^\star$, co-activation, and random removal are all farther from the documented intervention in margin distance.
Thus the completion and knockout tests distinguish branch matching from merely maximizing response magnitude or task damage.

\begin{table}[ht]
\centering
\small
\setlength{\tabcolsep}{7pt}
\caption{\textbf{Knockout relative to the documented backup intervention} (mean over four prompt seeds).
Each row removes the primary set plus eight selected heads.
Lower values indicate closer reproduction of the documented knockout; unrelated-text KL measures collateral change.}
\label{tab:oracleko}
\begin{tabular}{lccc}
\toprule
\rowcolor{hdr}
top-up & margin distance $\downarrow$ & output KL $\downarrow$ & unrelated KL $\downarrow$ \\
\midrule
documented backups \emph{(oracle)} & $0.00$ & $0.00$ & $0.09$ \\
\rowcolor{ourrow}
\coax & \best{$0.99$} & $0.17$ & $0.14$ \\
amplification-ratio variant & \best{$0.99$} & \best{$0.07$} & \best{$0.09$} \\
conditional energy & $1.71$ & $0.46$ & $0.47$ \\
co-activation & $1.79$ & $0.47$ & $0.15$ \\
\coax next-$8$ & $2.03$ & $1.38$ & $0.86$ \\
AtP$^\star$ GradDrop & $2.70$ & $1.02$ & $0.47$ \\
random & $1.24$ & $0.47$ & $0.24$ \\
\bottomrule
\end{tabular}
\end{table}

Paired bootstrap intervals over the $384$ prompt-level observations support the margin-distance ordering: relative to \coax, co-activation is $+0.799$ farther from the documented knockout ($95\%$ CI $[+0.655,+0.941]$), AtP$^\star$ is $+1.712$ farther ($[+1.545,+1.883]$), and random is $+0.248$ farther ($[+0.158,+0.336]$).

\clearpage
\section{Conditional Repair Beyond Documented Backups}
\label{app:panel-full}

This appendix expands the two generalization tests in Section~\ref{sec:panel}: repair alignment across mechanism clusters and held-out circuit completion across model families.

\subsection{Repair Prediction Across Mechanisms}
\label{app:panel}

The intervention panel contains twelve instances formed by three surface templates crossed with four primary sets: name-mover, S-inhibition, induction, and duplicate-token. Name-mover and S-inhibition instances use $96$ IOI prompts for scoring and a disjoint set of $96$ prompts for the freezing-based target. Induction and duplicate-token instances analogously use disjoint sets of $32$ repeated-token sequences. The former use the IOI margin and the latter repeated-token induction behavior, so each causal target matches the behavior disrupted by its primary set.

The causal target is the continuous quantity used in Section~\ref{sec:panel}, $\mathrm{repair}(u\mid\Sset)=m(\text{ablate }\Sset)-m(\text{ablate }\Sset,\text{ freeze }u)$. Agreement is measured by Spearman correlation across candidate heads. Signed conditional growth is positive in $11/12$ held-out instances and all four mechanism clusters.

\subsubsection{Hierarchical inference}
\label{app:hier}

The three templates within each primary set share the same model, head inventory, and mechanism, so we resample primary sets first and templates within them. The mechanism means are $+0.030$ for duplicate-token, $+0.193$ for induction, $+0.204$ for name-mover, and $+0.066$ for S-inhibition. Their equally weighted macro mean is $+0.123$, with nested $95\%$ CI $[+0.047,+0.201]$. Dropping any one mechanism leaves the macro estimate between $+0.096$ and $+0.154$, so the positive agreement is not carried by a single cluster.

\begin{table}[ht]
\centering
\small
\setlength{\tabcolsep}{5pt}
\caption{\textbf{Agreement of \coax with mechanism-matched causal repair} (Spearman $\rho$). Each mechanism value averages its three templates; the macro column weights the four mechanisms equally, and the interval resamples mechanisms first and templates within them.}
\label{tab:hier}
\begin{tabular}{lcccccc}
\toprule
\rowcolor{hdr}
score & dup.-token & induction & name-mover & S-inhib. & macro & nested $95\%$ CI \\
\midrule
\rowcolor{ourrow}
\coax & $+0.030$ & $+0.193$ & $+0.204$ & $+0.066$ & \best{$+0.123$} & $[+0.047,+0.201]$ \\
\bottomrule
\end{tabular}
\end{table}

\subsection{Conditional Circuit Completion Across Models}
\label{app:gen-group}

We test the set-level claim behind Figure~\ref{fig:regime}b under a held-out, equal-budget protocol. For each non-GPT-2 model, detection, selector calibration, and completion evaluation use disjoint repeated-token sequence sets: $32$ sequences identify the top four induction heads as the primary set, every selector receives the same $16$ calibration sequences, and completion is measured on $64$ held-out sequences. No backup labels are used. Each selector contributes its top $10$ remaining heads. Fisher scoring uses the full vocabulary through Gemma-2-2B and the fixed clean top-$4096$ support for Qwen2.5-7B, OLMo-2-7B, and Llama-3.1-8B. If $D(C)=m(\varnothing)-m(\Sset\cup C)$ is the held-out induction-log-probability drop after adding completion $C$ to the primary set, the reported completion gain is $D(C)/\mathbb{E}_{C_{\rm rnd}}D(C_{\rm rnd})$, where the expectation is over five independently sampled, size-matched random completions evaluated on the same held-out sequences.

Table~\ref{tab:indcross} compares \coax with intact-state single ablation, co-activation, AtP, AtP$^\star$ GradDrop, and the role-matched next induction heads (``own''). \coax exceeds matched random on all eight models and has a family-macro gain of $6.70$; it also gives the larger completion than the role-matched own-head control on five of eight models. The $8/8$ above-random pattern is unchanged on an absolute rather than normalized scale: the smallest absolute gain is $+1.97$ log-probability, $16.5\times$ the corresponding random control's draw-to-draw standard deviation. Excluding the largest-gain model still leaves a $3.55\times$ family-macro gain over matched random, so the transfer result is not carried by that model alone.

\begin{table}[ht]
\centering
\scriptsize
\setlength{\tabcolsep}{4.0pt}
\caption{\textbf{Held-out, equal-budget cross-model conditional circuit completion} ($k=10$). Entries are completion gain over a size-matched random completion; $1$ therefore corresponds to the random reference. Detection, calibration, and evaluation use disjoint sequence sets, and every selector receives the same $16$ calibration sequences.}
\label{tab:indcross}
\begin{tabular}{lcccccc}
\toprule
\rowcolor{hdr}
model & \coax & single & co-act & AtP & AtP$^\star$ & own \\
\midrule
Pythia-160M               & 1.88 & 2.04 & 1.95 & 1.80 & 1.89 & 1.91 \\
Pythia-410M               & 6.57 & 6.56 & 4.57 & 4.41 & 4.41 & 4.87 \\
Pythia-1.4B               & 4.98 & 3.94 & 3.49 & 2.62 & 3.76 & 3.74 \\
GPT-Neo-1.3B              & 4.60 & 8.76 & 3.99 & 7.68 & 7.68 & 6.94 \\
Gemma-2-2B                & 2.99 & 2.74 & 0.82 & 2.12 & 2.12 & 2.50 \\
Qwen2.5-7B-Instruct       & 22.46 & 3.15 & 1.46 & 22.36 & 22.36 & 3.68 \\
OLMo-2-7B                 & 2.80 & 2.55 & 2.35 & 2.31 & 3.52 & 4.05 \\
Llama-3.1-8B-Instruct     & 2.89 & 3.82 & 3.82 & 2.11 & 2.11 & 1.84 \\
\midrule
\rowcolor{ourrow}
family macro              & 6.70 & 4.20 & 2.63 & 6.59 & 6.86 & 3.75 \\
\bottomrule
\end{tabular}
\end{table}

\clearpage
\section{Limitations, Broader Impact, Ethics, and Reproducibility}
\label{app:limitations}

\subsection{Limitations}

Documented backup labels are available only for GPT-2-small IOI~\citep{wang2022ioi}, so ground-truth backup recovery is evaluated in that setting; generalization beyond it is assessed through label-free causal targets across mechanisms and architectures.
\coax conditions on a supplied primary set; automatically identifying that set is a separate problem, so the present work evaluates conditional circuit completion rather than joint primary-and-backup discovery.
Our main experiments use attention heads, matching the granularity of the documented backup circuit, although the formulation extends to other additive residual units and Appendix~\ref{app:ffn} provides an FFN-group implementation check.

\subsection{Broader impact}

\coax is an analysis method for frozen models and does not train or modify model weights.
Its intended benefit is more faithful mechanistic auditing: an intact-state analysis can appear to localize or remove a behavior while self-repair reroutes computation around the intervention.
Exposing such conditional routes can reduce false claims of circuit localization or removal.
The method does not itself add model capabilities, although, as with mechanistic interpretability generally, improved circuit knowledge may also inform downstream model analysis or editing.

\subsection{Assets, licenses, and ethics}

All experiments use publicly released pretrained checkpoints under their respective licenses: GPT-2 and GPT-Neo (MIT), Pythia, OLMo-2, and Qwen2.5-7B-Instruct (Apache-2.0), Gemma-2 (Gemma Terms of Use), and Meta-Llama-3.1-8B-Instruct (Llama 3.1 Community License).
All checkpoints are used unmodified for inference-only analysis.
Evaluation inputs comprise templated IOI prompts from \citet{wang2022ioi}, procedurally generated repeated-token sequences for induction, and WikiText-2 windows~\citep{merity2017wikitext} for unlabeled calibration.
No human subjects, private data, or personally identifying information are involved; the IOI templates contain only common given names.

\subsection{Reproducibility}
\label{app:repro}

All model weights remain frozen and no training is performed at any stage.
Appendix~\ref{app:method} specifies the complete \coax computation and implementation, Table~\ref{tab:hparams} records the fixed protocol constants and random seeds, and Appendix~\ref{app:setup} defines the checkpoints, data, metrics, and baseline implementations.
Measured wall-clock time and peak memory are reported in Table~\ref{tab:scal}; the main experiments require only a single GPU and run from seconds to minutes per scan.
The core GPT-2-small IOI discovery pipeline and its baselines also run on CPU in minutes, allowing the headline recovery result to be reproduced without GPU hardware.

\end{document}